\documentclass{article}
\usepackage{ijcai15}

\usepackage{times}
\usepackage{helvet}
\usepackage{courier}

\usepackage{url}

\usepackage{xspace}
\usepackage{amsfonts}
\usepackage{amsmath}
\usepackage{amssymb}
\usepackage{graphicx}

\usepackage[all]{xy}
\usepackage{textcomp}
\usepackage{paralist}

\usepackage{amsmath}
\usepackage{amssymb}
\usepackage{amsthm}

\usepackage[shortlabels]{enumitem} 

\usepackage{thmtools}
\usepackage{thm-restate}
\usepackage{tikz}
\usetikzlibrary{arrows,decorations.pathmorphing,backgrounds,positioning,fit}

\usepackage[algo2e,ruled,linesnumbered,lined, noend]{algorithm2e}
\usepackage{algorithm}
\usepackage{algorithmicx}

\newcommand\AlgPhase[1]{%
\begin{algorithmic}[0]
\vspace{-0.7ex}
\vspace*{-.1\baselineskip}\Statex\hspace*{\dimexpr-2pt\relax}\rule{0.9\textwidth}{0.4pt}%
\Statex\hspace*{-1pt}\textbf{#1}%
\vspace*{-.5\baselineskip}\Statex\hspace*{\dimexpr-2pt\relax}\rule{0.9\textwidth}{0.4pt}%
\end{algorithmic}
}

\newcommand\AlgPhaseTop[1]{%
\begin{algorithmic}[0]
\vspace{-0.7ex}
\Statex\hspace*{-1pt}\textbf{#1}%
\vspace*{-.5\baselineskip}\Statex\hspace*{\dimexpr-2pt\relax}\rule{0.9\textwidth}{0.4pt}%
\end{algorithmic}
}

\newcommand{\citeA}[1]{\citeauthor{#1} (\citeyear{#1})}

\newcommand{\dat}[2]{\Xi_#2(#1)}
\newcommand{\chase}[2]{\mathcal{H}_{#1,#2}}
\newcommand{\freeze}[1]{[{#1}]}
\newcommand{\query}[1]{Q^{#1}}
\newcommand{\str}[1]{\mathcal{#1}}

\newcommand{\edge}{\ensuremath{edge}}

\newcommand{\eset}{\emptyset}
\newcommand{\equality}{\approx}
\renewcommand{\epsilon}{\varepsilon}
\renewcommand{\phi}{\varphi}
\newcommand{\set}[1]{\{#1\}} %
\newcommand{\incl}{\subseteq} %
\newcommand{\homo}{\ensuremath{\hookrightarrow}\xspace}

\newcommand{\Anonymise}{\ensuremath{\mathtt{AddUnPredicates}}\xspace}
\newcommand{\AddRoles}{\ensuremath{\mathtt{AddBinPredicates}}\xspace}
\newcommand{\CheckRole}{\ensuremath{\mathtt{CheckRole}}\xspace}

\newcommand{\conc}[2]{\ensuremath{\mathtt{conc}_{#1}(#2)}\xspace}

\newcommand{\UOptViewC}[1]{\V^1_c}
\newcommand{\UOptViewE}[1]{\V^1_\exists}
\newcommand{\BOptView}[1]{\V^2}

\newcommand{\cert}[3]{\mathsf{cert}(#1,#2,#3)}

\newcommand{\cens}{\ensuremath{\mathsf{cens}}\xspace}
\newcommand{\vcens}[2]{\mathsf{vcens}_{#1}^{#2}}
\newcommand{\ocens}[2]{\mathsf{ocens}_{#1}^{#2}}

\newcommand{\Fcens}[1]{\ensuremath{\mathsf{Th}_{#1}}}

\newcommand{\true}{\ensuremath{\mathtt{True}}\xspace}
\newcommand{\false}{\ensuremath{\mathtt{False}}\xspace}

\newcommand{\View}{\ensuremath{\mathcal{V}}\xspace}

\newcommand{\pseudo}{\text{pseudo-obstruction}\xspace}

\newcommand{\A}{\ensuremath{\mathcal{A}}\xspace} 
\newcommand{\B}{\ensuremath{\mathcal{B}}\xspace} 
 
\newcommand{\D}{\ensuremath{\mathcal{D}}\xspace}

\renewcommand{\H}{\ensuremath{\mathcal{H}}\xspace}  
 
\newcommand{\I}{\ensuremath{\mathcal{I}}\xspace} 
\newcommand{\J}{\ensuremath{\mathcal{J}}\xspace}

\renewcommand{\P}{\ensuremath{\mathcal{P}}\xspace}

\renewcommand{\S}{\ensuremath{\mathcal{S}}\xspace}
\newcommand{\Inst}{\ensuremath{\mathbf{I}}\xspace}

\newcommand{\V}{\ensuremath{\mathcal{V}}\xspace}
\newcommand{\X}{\ensuremath{\mathcal{X}}\xspace}

\newcommand{\Oex}{\ensuremath{\mathcal{O}_{\mathsf{ex}}}\xspace}

\newcommand{\an}{\ensuremath{\mathit{an}}\xspace}
\newcommand{\john}{\ensuremath{\mathsf{John}}\xspace}
\newcommand{\bob}{\ensuremath{\mathsf{Bob}}\xspace}
\newcommand{\mary}{\ensuremath{\mathsf{Mary}}\xspace}
\newcommand{\movie}{\ensuremath{\mathsf{Seven}}\xspace}

\newcommand{\friendOf}{\ensuremath{\mathit{FoF}}\xspace}
\newcommand{\likes}{\ensuremath{\mathit{Likes}}\xspace}
\newcommand{\suspense}{\ensuremath{\mathit{Susp}}\xspace}
\newcommand{\crime}{\ensuremath{\mathit{Cr}}\xspace}
\newcommand{\thriller}{\ensuremath{\mathit{Thr}}\xspace}
\newcommand{\film}{\ensuremath{\mathit{Movie}}\xspace}
\newcommand{\thrillerFan}{\ensuremath{\mathit{ThrFan}}\xspace}
\newcommand{\movieFan}{\ensuremath{\mathit{MovFan}}\xspace}

\newcommand{\Dex}{\ensuremath{\D_{\mathsf{ex}}}\xspace}

\newcommand{\Vex}{\ensuremath{\View_{\mathsf{ex}}}\xspace}

\newcommand{\Uex}{\ensuremath{U_{\mathsf{ex}}}\xspace}

\newcommand{\Const}{\sigma}

\newcommand{\frontier}[1]{\ensuremath{\mathsf{fr}(#1)}\xspace}	%

\newcommand{\goals}[1]{\ensuremath{\mathbb{Q}}\xspace}

\newtheorem{theorem}{Theorem}

\newtheorem{lemma}[theorem]{Lemma}

\newtheorem{definition}[theorem]{Definition}
 
\newtheorem{myexample}[theorem]{Example}
\let\oldmyexample\myexample
\renewcommand{\myexample}{\oldmyexample\normalfont}
\newenvironment{example}{\begin{myexample}}{\hfill$\lozenge$\end{myexample}}

\renewcommand{\O}{\ensuremath{\mathcal{O}}\xspace} 

\begin{document}
\title{
Controlled Query Evaluation 
for 
Datalog and 
OWL 2 Profile Ontologies \\ (Extended Version)
\thanks{Work supported by the Royal Society,
the EPSRC grants
Score!, DBonto, and $\text{MaSI}^3$,
and the FP7 project \mbox{OPTIQUE}.}
}

\author{
Bernardo Cuenca Grau,
Evgeny Kharlamov,
Egor V. Kostylev,
Dmitriy Zheleznyakov\\
Department of Computer Science, University of Oxford, UK\\
\set{\textit{f\_name.l\_name}}@cs.ox.ac.uk
} 

\maketitle

\begin{abstract}
We study confidentiality enforcement in ontologies under the Controlled Query Evaluation 
framework, where a policy specifies the sensitive information and a censor ensures that query
answers that may compromise the policy are not returned. We focus on censors that ensure confidentiality 
while maximising information access, and consider both Datalog and the OWL 2 profiles as ontology languages.

\end{abstract}

\section{Introduction}
\label{sec:introduction}
 
As semantic technologies 
are becoming increasingly mature,
there is a  need for mechanisms to
ensure that confidential data
is only accessible by  authorised users.

Controlled Query Evaluation (CQE) 
is a prominent 
confidentiality enforcement framework,
in which 
sensitive information is 
declaratively specified 
by means of a \emph{policy}
and confidentiality is enforced by a \emph{censor}.
When given a query, 
a censor checks whether returning
the answer 
may lead to
a policy violation, 
in which case 
it returns a distorted answer. 
The CQE framework was introduced in 
\cite{DBLP:journals/tods/SichermanJR83}, and studied
in 
\cite{DBLP:journals/dke/BiskupB01,%
	DBLP:journals/ijisec/BiskupB04,%
	DBLP:journals/tkde/BonattiKS95,DBLP:journals/ijisec/BiskupW08}
for propositional databases. 
It has been recently extended to ontologies, where different formalisations have been proposed
\cite{DBLP:conf/semweb/BonattiS13,DBLP:conf/semweb/GrauKKZ13,DBLP:journals/tdp/StuderW14}.

We study CQE for ontologies that are expressed in the
rule language Datalog
as well as in the lightweight 
description logics (DLs) 
underpinning the standadised profiles of OWL~2~\cite{OWL2profiles}. 
We assume that data is hidden, and users
access the system by a query interface.
An ontology, which is known to users, 
provides the vocabulary and
background knowledge needed for users to formulate
queries, as well as to enrich query answers with 
implicit information. 
Policies, formalised as conjunctive queries, are
available only to system administrators, but not to ordinary users.  
The role
of the censor is to preserve confidentiality by 
filtering out 
those answers to user queries that could lead to a policy violation.  
In this setting, there is a
danger that confidentiality enforcement 
may over-restrict the access of the user.	
Thus, we focus on \emph{optimal} censors, 
which maximise answers to queries while 
ensuring confidentiality of the policy. 

We are especially interested
in censors that can be realised by off-the-shelf
reasoning infrastructure. 
To fullfil this requirement,
we introduce in Section~\ref{sec:censor-types} 
\emph{view} and \emph{obstruction} censors.
View censors return only answers that
follow from 
the ontology and an anonymised dataset (a view)
where some occurrences of constants may have been replaced with labelled nulls.
The censor
answers faithfully all queries against the view; thus,
any information not captured
by the view is inaccessible by default.
View censors may require materialisation of implicit data, 
and hence are well-suited 
for applications where 
materialisation is feasible.
Obstruction censors are defined by a set of
``forbidden query patterns'' (an obstruction), where
all answers instantiating such patterns are
not returned to users.
These censors do not require data modification 
and are well-suited for applications such as
Ontology Based Data Access (OBDA), 
where data is managed by an 
RDBMS.
Obstruction censors are dual to view censors in the sense
that they specify the information that users are denied access to.
We formally characterise this duality, and show
that their capabilities
are incomparable.

In Section~\ref{sec:view-censors} we
investigate the limitations of view censors and 
show that checking existence of an optimal view
is undecidable for Datalog ontologies.
We then study fragments of Datalog for which optimal views always 
exist and  extend our results
to OWL 2 profile ontologies.
In Section~\ref{sec:obstruction-optimal} 
we focus on obstruction censors,
and provide sufficient and necessary conditions for
an optimal censor to exist.
Then, 
we propose a tractable algorithm
for computing optimal obstruction censors for
linear Datalog ontologies and apply our results to 
OWL 2 QL ontologies. 

\section{Preliminaries}

\begin{table}[t]
\begin{small}
$\ \ \
\mathrm{(1)} 
\;%
A(x) {\wedge} R(x,y_1) {\wedge} B(y_1) {\wedge} R(x,y_2) {\wedge} B(y_2) \rightarrow y_1 \equality y_2,
$\\[-3.5ex]
\begin{align*}
&\mbox{}\hspace{-.8ex} \mathrm{(2)} \;
R(x,y) \rightarrow S(x,y),
&& 
\hspace{-1ex}
\mathrm{(3)} \;
A(x) \rightarrow \exists y. [R(x,y) {\wedge} B(y)],
\\
&\mbox{}\hspace{-.8ex} 
\mathrm{(4)} \;
		 A(x) \rightarrow x \equality a,
&&
\hspace{-1ex}
\mathrm{(5)}  \;
	R(x,y) {\wedge} S(y,z) \rightarrow T(x,z),
\\
&\mbox{}\hspace{-.8ex}  \mathrm{(6)} \;
		 A(x) {\wedge} B(x) \rightarrow C(x), 
&&
\hspace{-1ex}
\mathrm{(7)} \;
	 	 A(x) \wedge R(x,y) \rightarrow B(y),
\\
&\mbox{}\hspace{-.8ex}  \mathrm{(8)} \;
	R(x,y) \rightarrow S(y,x),
&&
\hspace{-1ex}
\mathrm{(9)} \;
		R(x,a) \rightarrow B(x),
\\ 
&\mbox{}\hspace{-.8ex}  \mathrm{(10)} \;
		R(x,y) \rightarrow A(y),
&&
\hspace{-1ex}
\mathrm{(11)} \;
		A(x) \rightarrow R(x,a),\\
&\mbox{}\hspace{-.8ex}  
\mathrm{(12)} \;
A(x) \rightarrow B(x),
&&
\hspace{-1ex}
\mathrm{(13)} \; 
		R(x,y) \wedge B(y) \rightarrow A(x).
\end{align*} 
\caption{OWL 2 profile axioms as rules}
\label{tab:DL-axioms}
\end{small}
\end{table}

We adopt standard notions
in first order logic over function-free finite signatures.
Our focus is on ontologies, so we assume
signatures with predicates of arity at most two. 
We treat equality $\equality$ as an ordinary predicate,
	but assume that any set of 
	formulae containing $\equality$ also contains all 
	the axioms of $\equality$ for its signature.

\smallskip
\noindent\textbf{Datasets and Ontologies}
A \emph{dataset} is a finite set of facts (i.e., ground atoms). 
An \emph{ontology} is a finite set of \emph{rules}, that is, formulae of the form
$$
\varphi(\vec{x}) \rightarrow \exists \vec{y}. \psi(\vec{x},\vec{y}),
$$
where the \emph{body} ${\varphi(\vec{x})}$ and the \emph{head} $\psi(\vec{x},\vec{y})$ are conjunctions of atoms, and variables $\vec{x}$ are implicitly universally quantified. %
We restrict ourselves to ontologies $\O$ and datasets $\D$ such that
$\O \cup \D$ is satisfiable, which 
ensures that answers to queries are meaningful. 
A rule is 
\begin{itemize}[leftmargin=10pt,noitemsep,topsep=1pt]
\item[--] \emph{Datalog} if the head has a single atom and $\vec y$ is empty;
\item[--] \emph{guarded} if the body has an atom (\emph{guard}) with all $\vec{x}$;
\item[--] \emph{linear} if the body has a single atom;
\item[--] \emph{multi-linear} if the body contains only guards;
\item[--] \emph{tree-shaped} if the undirected multigraph with an edge $\{t_1,t_2\}$ for each binary body atom $R(t_1,t_2)$ is a tree.
\end{itemize}
An ontology is of a type above if so are all the rules in it.

\smallskip
\noindent\textbf{OWL 2 Profiles}
Table~\ref{tab:DL-axioms} provides the types of rules 
sufficient to capture the axioms in the OWL 2 RL, EL, and QL profiles.
We treat the $\top$ concept in DLs as a unary predicate and 
assume that each ontology contains
the rule $S(\vec x) \rightarrow \top(x)$ 
for each predicate $S$ and variable $x$ from $\vec x$. 
An ontology consisting of rules in Table~\ref{tab:DL-axioms}  is 

\begin{itemize}[leftmargin=10pt,noitemsep,topsep=1pt]
\item[--] \emph{RL} if it has no rules of type (3);
\item[--] \emph{QL} if it only has rules of types~(2), (3), (8), (10), (12);
\item[--] \emph{EL} if it has no rules of types~(1),~(7),~(8). 
\end{itemize}

\smallskip
\noindent\textbf{Queries}
A \emph{conjunctive query} (\emph{CQ}) with \emph{free} variables $\vec x$ is 
a formula $Q(\vec{x})$ of the form 
$\exists \vec{y}.\varphi(\vec{x},\vec{y})$, with the \emph{body} $\varphi(\vec{x},\vec{y})$
a conjunction of atoms. A \emph{union} of CQs (\emph{UCQ}) is disjunction of CQs with same free variables.
Queries with no free variables are \emph{Boolean}.
A tuple of constants $\vec a$ is a  (\emph{certain}) \emph{answer} 
to $Q(\vec{x})$ over ontology $\O$ and dataset $\D$ if ${\O \cup
\D \models Q(\vec{a})}$. The set of answers to $Q(\vec{x})$ over $\O$ and $\D$ is denoted by $\cert{Q}{\O}{\D}$.
\section{Basic Framework}\label{sec:basic-framework}

We assume that data $\D$ is hidden
while the ontology $\O$ is known to all users. 
It is assumed that system administrators are in charge of 
specifying policies as CQs, and that policies are assigned 
to users by standard mechanisms such as 
role-based access control \cite{DBLP:journals/computer/SandhuCFY96}. 
\begin{definition}
A \emph{CQE} \emph{instance} $\Inst$ is a triple $(\O,\D,P)$, with
$\O$ an ontology, $\D$ a dataset, and $P$ a CQ, which is called \emph{policy}.
The instance $\Inst$ is Datalog, guarded, etc.~if so is the ontology $\O \cup \{\varphi(\vec{x},\vec{y}) \to A_p(\vec x)\}$, where $\varphi(\vec{x},\vec{y})$ is the body of $P$ and $A_p$ 
a fresh predicate.
\end{definition}

\begin{example}\label{ex:running-example} 
Consider the following ontology and dataset that describe
an excerpt of a social network:
\begin{align*}
& \begin{array}{l}
\!\!\!\Oex = \set{\likes(x,y) \land \thriller(y) \rightarrow \thrillerFan(x), \\
	\!\suspense(x) {\land} \crime(x) \rightarrow \thriller(x), \friendOf(x,\!y) \rightarrow \friendOf(y,\!x)
		 },
\end{array}\\
& \begin{array}{l}
\!\!\!\Dex \!=\! \set{\friendOf(\john, \!\bob), \friendOf(\bob, \!\mary), \crime(\movie), \\
		 \!\likes(\john, \!\movie), \likes(\bob, \!\movie), \suspense(\movie)}.
\end{array}
\end{align*}
Here, the ontology \Oex  states, for example, that people who like thrillers are thriller fans, or that
friendship is a symmetric relation.
Then, a policy $P_{\mathsf{ex}} = \friendOf(\john,x)$ forbids access to John's friend list.
\end{example}
A key component of a CQE system 
is the \emph{censor}, whose goal is to decide according to the policy
which query answers can be safely returned to users.

\begin{definition}
A \emph{censor} for a CQE instance $(\O,\D, P)$ is a function 
$\cens$ mapping each CQ $Q$ to a subset of $\cert{Q}{\O}{\D}$.
The \emph{theory} $\Fcens{\cens}$
of $\cens$ is the set
\begin{equation*}
\{Q(\vec a) \mid \vec a \in \cens(Q) \text{ and } Q(\vec x) \text{ is a CQ} \}.
\end{equation*}
Censor $\cens$ is \emph{confidentiality preserving} if for each tuple $\vec a$
of constants  $\O \cup \Fcens{\cens} \not\models P(\vec a)$.
It is \emph{optimal} if 
\begin{itemize}[leftmargin=10pt,noitemsep,topsep=1pt]
\item[--] it is confidentiality preserving, and
\item[--] no 
confidentiality preserving censor $\cens' \neq \cens$
exists such that $\cens(Q) \incl \cens'(Q)$ for every CQ $Q$.
\end{itemize}
\end{definition}
Intuitively, $\Fcens{\cens}$ represents all the information that
 a user can gather by asking CQs to
 the system. 
If the censor is confidentiality preserving, 
then no information can be obtained about the policy, 
regardless of the number of CQs asked. 
In this way, optimal censors maximise information 
	accessibility without
	compromising the policy.
\section{View and Obstruction Censors} \label{sec:censor-types}
The idea behind \emph{view censors} is to 
modify the dataset by anonymising occurrences of
constants as well as by adding or removing facts, whenever needed.
We refer to such modified dataset as an \emph{(anonymisation) view}.
The censor returns only  
the answers that follow from the ontology and view;
in this way, the main workload of the censor 
amounts to the computation of certain answers, which
can be delegated to the query answering engine.

\begin{definition}\label{def:view-censor}
A \emph{view} $\View$ for $\Inst = (\O,\D,P)$ is a dataset over the signature of $\Inst$
extended with a set of fresh constants. 
The \emph{view censor} $\vcens{\Inst}{\View}$
is the function mapping each CQ $Q(\vec x)$ to the set $\cert{Q}{\O}{\D} \cap \cert{Q}{\O}{\View}$.
The view is \emph{optimal} if so is its corresponding censor.
\end{definition}

Clearly, for the censor
to
be confidentiality preserving $\O \cup \View$ must not entail
any answer to the policy. On the other hand, to ensure optimality a view must encode
 as much information
from the hidden dataset as possible.

\begin{example}
\label{ex:view-basic}
Consider the view $\Vex$ obtained from $\Dex$ 
	by replacing \bob with a fresh $\an_b$. 
Intuitively, $\Vex$ 
	is the
	result of ``anonymising'' the constant \bob, 
	while keeping the structure of the data intact.
Since $\Vex$ contains no information about $\bob$, 
	we have $\cert{P_{\mathsf{ex}}}{\Oex}{\Vex} = \emptyset$ and
	the censor based on $\Vex$ is confidentiality preserving. 
View $\Vex$, however, is not optimal: for instance,
	$\Oex \cup \Vex$ does not entail the fact $\likes(\bob,\movie)$, 
	which can be added to the view without
	violating confidentiality.
\end{example}

The idea behind \emph{obstruction censors} is to associate
to a CQE instance a Boolean UCQ $U$ s.t.\  
the censor returns an answer $\vec a$  to 
a CQ $Q(\vec x)$ only if no CQ in $U$ follows from
$Q(\vec a)$. Thus, the obstruction can be seen as a set of
forbidden query patterns, which should not be disclosed.
 
\begin{definition}
An \emph{obstruction} $U$ for $\Inst = (\O,\D,P)$ is a Boolean UCQ.
The \emph{obstruction censor} $\ocens{\Inst}{U}$
 based on $U$ 
is the function that maps each CQ $Q(\vec x)$ to the set
\begin{align*}
\{ \vec a \mid \vec a \in \cert{Q}{\O}{\D} \text{ and } Q(\vec a) \not\models U \}.
\end{align*}
The obstruction is \emph{optimal} if so is its censor $\ocens{\Inst}{U}$.
\end{definition}
Similarly to view censors, obstruction censors do not require
dedicated algorithms: checking 
$Q(\vec a) \models U$ 
can be delegated to an RDBMS. Obstructions can be virtually maintained
and do not require data materialisation.

\begin{example}\label{ex:obstruction-basic}
The censor based on $\Vex$ from Example \ref{ex:view-basic}
can also be realised with the following obstruction $\Uex$:\\[-3ex]
\begin{align*}
	& \exists x. \friendOf(x, \bob) \lor \exists x. \friendOf(\bob, x) \lor{}\\
	& \exists x. \likes(\bob, x) \lor \thrillerFan(\bob).
\end{align*}
Intuitively, $\Uex$ ``blocks'' query answers involving $\bob$; and all other answers are the same as over $\Oex \cup \Dex$.  
\end{example}
Examples \ref{ex:view-basic} 
and \ref{ex:obstruction-basic} show that the same censor 
may be based on both a view and an obstruction.
These censors, however, behave \emph{dually}:
a view explicitly encodes the information accessible to users,
whereas obstructions specify information which users are
denied access to. 
It is not
obvious whether (and how)
a view can be realised by an obstruction, or vice-versa.
We next focus on Datalog ontologies and
characterise when a view $\View$
and obstruction $U$ yield the same censor. 
We start with few definitions. 

Each Datalog ontology
$\O$ and dataset $\D$ have a unique \emph{least Herbrand model}
$\chase{\O}{\D}$: a finite structure
satisfying $\vec a \in \cert{Q}{\O}{\D}$ iff $\chase{\O}{\D} \models Q(\vec a)$
for every CQ $Q$. Thus, this model captures all the
information relevant to
CQ answering.
A natural specification of the duality between 
views and obstructions is then as follows: 
$U$ and $\View$
implement the same censor if and only if $U$ captures 
the structures
\emph{not} homomorphically embeddable into  $\chase{\O}{\V}$.
To formalise this statement, 
we recall the central problem in the (non-uniform)
constraint satisfaction theory.

\begin{definition}[Kolaitis and Vardi, 2008]
\nocite{DBLP:conf/dagstuhl/KolaitisV08}
Let $\mathbb{C}$ be a class of finite structures
and let $\mathbb{C}'$ be a subset of $\mathbb{C}$. First-order sentence $\psi$ \emph{defines} $\mathbb{C}'$ if $\str{I} \in \mathbb{C}'$ is equivalent to $\str{I} \models \psi$ for every 
structure $\str{I} \in \mathbb{C}$. 
\end{definition}

Let $\str{J} \homo \str{J'}$ denote the fact that there is a homomorphism from a structure $\str{J}$ to a structure $\str{J'}$. The correspondence is given in the following theorem.

\begin{restatable}{theorem}{censorsimulation}\label{th:simulation}
Let $\Inst = (\O,\D,P)$ be Datalog and 
$\mathbb{C} = \{\str{I} \mid \str{I} \text{ finite}, \str{I} \homo \chase{\O}{\D}\}$.
Then, $\vcens{\Inst}{\View} = \ocens{\Inst}{U}$ iff $U$ defines the set $\mathbb{C} \setminus \{ \str{I} \in \mathbb{C} \mid \str{I} \homo \chase{\O}{\V} \}$. %
\end{restatable}
Using this theorem together with 
definability results in Finite Model Theory, 
we can show that views and obstructions cannot simulate
one another in general.

\begin{restatable}{theorem}{censorNotSimulationGeneral}\label{th:no simulation in general}
There is a Datalog CQE instance admitting a confidentiality preserving view censor that is not based on any obstruction.
Conversely, there is a Datalog CQE instance admitting a confidentiality preserving obstruction censor that is not based on any view.
\end{restatable}

\section{Optimal View Censors}\label{sec:view-censors}
 
Our discussion
 in Section~\ref{sec:censor-types} suggests that
view and obstruction censors must be studied independently. In this section we
focus on view censors and start
by establishing their theoretical limitations.
The following example shows that optimal view censors may not exist, even 
if we restrict ourselves to empty ontologies.

\begin{example}
\label{ex:no optimal view}
Consider a CQE instance with empty ontology, 
	dataset consisting of a fact $R(a,a)$, and policy $P = \exists x\, \exists y\, \exists z. 
	R(x,y) \wedge R(y, z) \wedge R(z, x)$.
Consider also the family of Boolean CQs $Q_n = \exists x_1\ldots\exists x_n. \bigwedge_{i < j} R(x_i,x_j)$, which
intuitively represent strict total orders on $n$ elements.
Answering these queries positively is harmless: $\V \cup \set{Q_n}_{n \geq 1} \not\models P$ 
	for any confidentiality preserving view $\V$.
Assume now that $\V$ is optimal, and let $m$
be the number of constants in $\V$.
Then, $\V \not\models Q_{m+1}$
	since otherwise $\V$ would encode a self-loop and
	violate the policy.
This contradicts the optimality of $\V$, and hence no optimal view exists.
\end{example} 

Furthermore, determining the existence of
an optimal view
is undecidable even for Datalog CQE instances.

\begin{restatable}{theorem}{undecTheorem}
The problem of checking whether a  Datalog CQE instance 
admits an optimal view is undecidable.
\end{restatable}

\begin{proof}[Proof (idea)]
The proof is by reduction to the undecidable problem 
of checking whether a deterministic Turing 
machine without a final state has a 
repeated configuration in a run 
on the empty tape. 
For each such machine we construct a CQE
instance such that the run corresponds to an infinite grid-like 
``view'' with axes for the tape and time. 
The ontology guarantees that representations of adjacent configurations
agree with the transition function, and the policy forbids 
invalid configurations (e.g., 
with many symbols in a cell). 
Coinciding configurations appear in the run iff 
the grid can be ``folded'' to a finite view on all sides (e.g., 
if configurations can be merged).  
\end{proof} 

In what follows, 
we identify classes of CQE instances that guarantee existence of optimal view censors. 
We start by studying restrictions on Datalog ontologies and then adapt the obtained results to the OWL 2 profiles.

\subsection{Guarded Tree-Shaped Datalog}
\label{sec:data-view-RL}

The idea behind view censors
is to anonymise information in the original data in such a way that
the policy cannot be violated. 
For instance, in Example~\ref{ex:view-basic}
we substituted the atom $\friendOf(\john, \bob)$
with $\friendOf(\john, \an_b)$, where $\an_b$ is a fresh constant
that is an anonymised copy of  $\bob$.
In general, however, many such anonymous copies may be required for each
data constant to encode all the information required for
ensuring optimality. The limit case is illustrated by Example~\ref{ex:no optimal view}, where no finite number of
fresh constants suffices for optimality.

Observe that the CQE instance used in Example \ref{ex:no optimal view} is neither guarded nor tree-shaped due to the form of the policy. In what follows, we show
that an optimal view can always be constructed using at most exponentially many 
anonymous constants if we restrict ourselves to Datalog CQE instances that
are guarded and tree-shaped.

We next provide an intuitive idea of the construction.
Consider the view for a CQE instance $(\O,\D,P)$ consisting of the following three components $\View_1$--$\View_3$.
\begin{itemize}[leftmargin=17pt,noitemsep,topsep=1pt]
\item[\em(1)] Component $\V_1$ is any maximal set of unary atoms in 
$\chase{\O}{\D}$ that does not compromise the policy.
\item[\em(2)] To construct $\View_2$, we consider an anonymised copy $a_{\mathcal B}$ of each constant $a$ and each set
 $\mathcal B$ of unary predicates $B$ s.t.\ $\chase{\O}{\D} \models B(a)$. The corresponding set of all unary atoms $B(a_{\mathcal B})$ for $B \in \mathcal B$ is a part of $\View_2$ if and only if it is ``safe'', that is, neither discloses the policy nor entail new facts together with $\O \cup \View_1$. 
\item[\em(3)] Finally, $\View_3$ consists of a maximal set of binary atoms on all the constants (including the copies) that are justified by $\chase{\O}{\D}$ and do not disclose the policy.
\end{itemize}
Optimality of this view follows immediately from the construction.
The view, however, may require exponentially many anonymised copies of data constants.
The need for them is illustrated by the following example.

\begin{figure}[t]
\hrule
\medskip
\begin{center}
\begin{tikzpicture}
[RDFNode/.style={inner sep=1mm, rectangle, rounded corners=5pt, draw},
 RDFNodeA/.style={inner sep=1.5mm, rectangle, rounded corners=5pt, draw},
 RDFLabel/.style={inner sep=.5mm},
 >=stealth,
 scale=1]

\node at ( 0.5,1)     (John)  [RDFNodeA] {${}_{\left.\right.}\john_{\left.\right.}$};
\node at ( 0.5,0)   (Bob) [RDFNodeA] {\bob};

\node at ( 3,1)     (aJohn1)  [RDFNode] {$\john_{\{\movieFan\}}$};
\node at ( 6,1)     (aJohn2)  [RDFNode] {$\john_{\{\movieFan, \thrillerFan\}}$};
\node at ( 4.5,0)   (aBob) [RDFNode] {$\bob_{\{\movieFan, \thrillerFan\}}$};
\node at ( 2,0)   (dots) {$\cdots$};

\draw [->] (John) to node [RDFLabel, auto] {} (Bob);
\draw [->] (John) to node [RDFLabel, auto] {} (aBob);
\draw [->] (aJohn1) to node [RDFLabel, auto] {} (Bob);
\draw [->] (aJohn1) to node [RDFLabel, auto] {} (aBob);
\draw [->] (aJohn2) to node [RDFLabel, auto] {} (aBob);

\end{tikzpicture}
\end{center}
\hrule

\caption{Part of optimal view in Example~\ref{ex:algorithm-intuition} (omitted labels coincide to subscipts, arrows represent $\friendOf$)}
\label{fig:running example}
\end{figure}
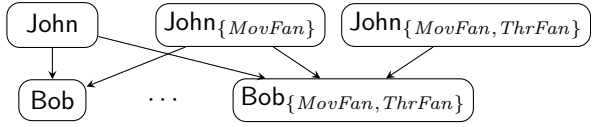

\begin{myexample}\label{ex:algorithm-intuition}
Consider the CQE instance with ontology consisting of rules
$\thrillerFan(x) {\rightarrow} \movieFan(x)$ and $\thrillerFan(y) {\land} \friendOf(x,y) \rightarrow$ $\movieFan(x)$, 
dataset consisting of facts $\friendOf(\john,\bob)$, $\thrillerFan(\john)$ and $\thrillerFan(\bob)$, and policy
$\movieFan(x)$.
The essential part of the optimal view obtained using the aforementioned
construction  
is given in Figure~\ref{fig:running example}.
Here $\View_1 = \eset$, $\View_2$
contains unary atoms over the anonymised copies $\john_{\{\movieFan\}}$ and $\john_{\{\movieFan, \thrillerFan\}}$ of $\john$,  and
$\bob_{\{\movieFan, \thrillerFan\}}$ of $\bob$,
while $\View_3$ contains the $\friendOf$ atoms represented by arrows.
Note that at least two anonymised copies of $\john$ are necessary
in any optimal view to answer correctly ``harmless'' queries such as
\begin{align*} 
  \exists x\, \exists y\, \exists z. &
		\thrillerFan(x) \land  %
		\friendOf(x,y) \land \thrillerFan(y) \land {} \\
		& %
		\friendOf(z,y) \land \movieFan(z) \land \friendOf(z, \bob). \quad \diamondsuit
\end{align*}
\end{myexample}

This example shows that, in order to avoid the exponential blow up in the number of anonymised copies,
we need further restrictions on the ontology.
In particular, in the case of multi-linear CQE instances we can guarantee that just one copy suffices for every constant.

The following theorem formalises the intuition above.

\begin{restatable}{theorem}{guardedTreeShapedDatalogView}\label{th:guarded}
Let $\Inst$ be a Datalog tree-shaped CQE instance.
If $\Inst$ is guarded, 
	it admits an optimal view that can be computed in time 
	exponential in $\vert \Inst \vert$ and polynomial in data size.
If $\Inst$ is multi-linear, 
	it admits an optimal view that can be computed in time 
	polynomial in $\vert \Inst \vert$.
Additionally, $\Inst$ has a unique optimal censor if it
is linear.
\end{restatable}

\subsection{OWL 2 Profiles}
\label{sec:view censors for EL and QL}
 
The result in Theorem \ref{th:guarded} is immediately applicable to
RL ontologies, with the only restriction that they do not contain rules
of types (1), (4), or (5) in Table \ref{tab:DL-axioms}.
In contrast to RL, the QL and EL profiles
provide means for capturing existentially quantified knowledge.
To bridge this gap, we show that every (guarded)
QL or EL CQE instance $\Inst = (\O,\D,P)$ can be polynomially 
trasformed into
a Datalog CQE instance $\Inst' = (\O',\D,P)$
by rewriting $\O$ into a (guarded and tree-shaped)
Datalog ontology $\O'$ in such a way that 
optimal views for 
$\Inst$ can be directly obtained
from those for $\Inst'$.
We start by specifying what constitutes an acceptable rewriting
$\O'$ of $\O$.

\begin{definition}
Let $\Const$ be a set of constants.%
A Datalog ontology $\O'$ is a $\Const$-\emph{rewriting} of an ontology $\O$
if $\cert{Q}{\O}{\D} = \cert{Q}{\O'}{\D}$ for each tree-shaped CQ $Q$ and dataset $\D$ over constants from $\Const$.
\end{definition} 

The following proposition provides the mechanism to
reduce optimal view computation for arbitrary ontologies
to the case of Datalog.

\begin{restatable}{proposition}{HornDatalog}
\label{prop:Horn2DatalogProp}
Let $\Inst  = (\O,\D,P)$ 
be a CQE instance over constants $\Const$ with $P$ tree-shaped,
 and 
$\O'$ a $\Const$-rewriting of $\O$
s.t.\ $\O' \models \O$. If $\View'$ is an optimal view for
$\Inst' = (\O', \D, P)$, then 
$\chase{\O'}{\View'}$ is an optimal view for $\Inst$.
\looseness=-1
\end{restatable}
With this proposition at hand, 
we just need to devise a technique for rewriting any
QL (or guarded EL) ontology 
into a stronger Datalog ontology, which, however, 
preserves the answers to all tree-shaped queries. To this end, 
we exploit
techniques developed for the so-called \emph{combined approach}
to query answering \cite{DBLP:conf/ijcai/KontchakovLTWZ11,DBLP:conf/ijcai/LutzTW09,DBLP:conf/semweb/LutzSTW13,DBLP:conf/aaai/StefanoniMH13}. The idea is to  
transform rules of type (3) into Datalog by
Skolemising existentially quantified variables into globally 
fresh constants. Such transformation strengthens the
ontology; however, if applied to a QL or guarded EL
ontology, it preserves answers to tree-shaped CQs for any dataset
over~$\Const$ \cite{DBLP:conf/aaai/StefanoniMH13}.
 
\begin{definition}
\label{Xidef}
Let $\O$ be an ontology and $\Const$ be a set of constants. The ontology
$\dat{\O}{\Const}$ is obtained from $\O$ by replacing each rule 
$A(x)  \rightarrow  \exists y. [R(x,y) \wedge B(y)]$
 with
$$
A(x) {\rightarrow} R'(x,a), R'(x, y) {\rightarrow}  R(x, y), R'(x, y) {\rightarrow} B(y),
$$ 
where $R'$ is a fresh binary predicate, uniquely associated to the original rule, and $a$ is a globally fresh constant not from $\Const$, uniquely associated to $A$ and $R$.
\end{definition}

\begin{theorem}\label{thm:combined-approach}
For any ontology $\O$ we have $\dat{\O}{\Const} \models \O$.
Furthermore, if $\O$ is either a QL or guarded EL ontology,
then $\dat{\O}{\Const}$ is a $\Const$-rewriting of $\O$.
\end{theorem}

Proposition~\ref{prop:Horn2DatalogProp} and Theorem \ref{thm:combined-approach}
ensure that $\chase{\dat{\O}{\Const}}{\View}$ is an optimal view for $\Inst$ 
whenever $\View$ is such a view for $\Inst' = (\dat{\O}{\Const},\D,P)$. The transformation of $\O$ to $\dat{\O}{\Const}$ 
preserves linearity, guardedness, and tree-shapedness, 
so the results of Section~\ref{sec:data-view-RL} are applicable to $\Inst'$.

\begin{theorem}\label{def:view-EL-QL-optimal}
Every guarded EL CQE instance  admits
an optimal view that can be computed in exponential time.
Every QL instance
admits a unique optimal censor, which is implementable by a view of
polynomial size.
\end{theorem}  %
\section{Optimal Obstruction Censors}
\label{sec:obstruction-optimal}

Similarly to Section~\ref{sec:view-censors}, we start the study of optimal obstruction censors with its limitations. 
The following example shows that such a censor may not exist even if we restrict ourselves to ontologies with 
only one  rule.
 
\begin{example}
	\label{example:no obstruction censor}
	Consider a CQE instance with ontology
$\set{R(x,y) \wedge A(y)\rightarrow A(x)}$, dataset
$\{R(a,a), A(a)\}$, and policy $A(a)$.
Let $Q_n$, $n > 0$, be a family of Boolean CQs
\begin{align*}
        \exists \vec x.\, R(a,x_1) \wedge R(x_1,x_2) \wedge \!\cdots \!\wedge
	R(x_{n-1},x_n) \wedge A(x_n).
\end{align*}
With the help of the ontology 
each of $Q_n$ 
discloses the policy.
Thus, 
each $Q_n$ should entail 
a Boolean CQ in any optimal obstruction.
Consider the set of all CQs that are entailed by queries $Q_n$ but not equivalent to any of them. On the one hand, this set is ``harmless'', than is, any obstruction censor should answer all these queries positively. On the other hand, the CQs $Q_n$ do not entail each other.
Hence, any optimal obstruction should contain a CQ equivalent to each $Q_n$, which is however not possible, because $n$ is unbounded.
\end{example}

We leave the question of decidability of checking the existence of an optimal obstruction for a CQE instance open. 
Answering this question positively would imply a solution to a long-standing open problem.
In Appendix~\ref{app:reduction}
we provide a reduction from the problem of uniform 
boundedness for binary Datalog, for which the decidability is unknown~\cite{DBLP:journals/siamcomp/Marcinkowski99}, 
to the existence problem of optimal obstructions for Datalog CQE instances. 
In the rest of the section we give a characterisation of optimal 
obstructions for Datalog instances in terms of resolution proofs 
and identify restrictions for which the characterisation guarantees 
existence of such obstructions.

\subsection{Characterisation of Optimal Obstructions}
\label{sec:characterisation-obstruction}
We first recall the
standard notion of SLD resolution. 

A \emph{goal} is a conjunction of atoms.
An \emph{SLD resolution step}
takes a goal $\beta \wedge \phi$ with a selected atom $\beta$  
and a sentence $r$ that is either a Datalog rule $\psi \rightarrow \delta$ or a fact $\delta$, and produces a new goal 
$\phi\theta \wedge \psi\theta$,
where $\theta$ is a most general unifier 
of $\beta$ and $\delta$ (assuming that $\psi$ is empty in the case when $r$ is a fact). An \emph{(SLD) proof} of a goal $G_0$ in a 
Datalog ontology $\O$ and dataset $\D$ is a sequence of goals
$G_0, G_1, \ldots, G_{n}$, where $G_n$ is empty, and each $G_{i}$ is obtained from $G_{i-1}$ and
a sentence (rule or fact) in $\O \cup \D$ by an SLD resolution step.

Resolution is sound and complete:
for any Datalog ontology $\O$, dataset $\D$,
and goal $G$ (such that $\O \cup \D$ is satisfiable) there is a proof of $G$ in $\O$ and $\D$
if and only if $\O \cup \D \models \exists^* G$ for 
 the existential closure $\exists^* G$~of~$G$.

We next characterise optimal obstructions using 
SLD proofs. %
Intuitively, if an obstruction censor answers positively
sufficient number of Boolean CQs $\exists^* G$
 for goals $G$ in a proof of a policy, then a user could
reconstruct (a part of) this proof 
and compromise the policy. 
Also, there can be many proofs,
and a user may compromise
the policy by reconstructing any of them.
Thus, 
to ensure that a censor is confidentiality preserving, 
we must guarantee that
the obstruction contains enough CQs 
to prevent reconstruction of any proof. 
If we also want the censor to be optimal, the obstruction should not block too many CQs.
As we will see later on, these requirements may be in conflict and
lead to an infinite ``obstruction''. 
Next definitions formalise this intuition.

\begin{definition}
\label{def:pseudo obstruction}
Let $\Inst = (\O,\D,P)$ be a Datalog CQE instance, $\goals{\Inst}$ be the set of all Boolean CQs $\exists^* G$ for
goals $G$ in proofs of $P(\vec a)$ in $\O$ and $\D$
for some tuple of constants $\vec a$, and $\mathbb S$ be a maximal subset of $\goals{\Inst}$ such that $\O \cup \mathbb S \not\models P(\vec a)$ for 
any $\vec a$.
Then, a \emph{\pseudo} for $\Inst$ is 
a subset of $\goals{\Inst} \setminus \mathbb S$ that contains a CQ $Q'$ for any $Q$ in $\goals{\Inst} \setminus \mathbb S$ with $Q \models Q'$.
\end{definition} 

 The next theorem establishes
 the connection between pseudo-obstructions 
 and optimality.

\begin{restatable}{theorem}{optimalUCQcensor}
\label{th:characterisation}
Let $\Inst$ be a Datalog CQE instance.
\begin{enumerate}[leftmargin=*,noitemsep,topsep=0pt]	
  \item If $\Upsilon$ is a finite \pseudo for $\Inst$, then 
  $\bigvee_{Q \in \Upsilon} Q$ is an optimal obstruction for $\Inst$.
  \item If each \pseudo for $\Inst$ is infinite, then no optimal obstruction
  censor for $\Inst$ exists.
\end{enumerate}
\end{restatable}
 
This theorem has implications on 
the expressive power of obstructions. In particular, we can now extend
the result in Theorem \ref{th:no simulation in general}, 
which applies to censors that are not necessarily optimal, to capture also optimality.

\begin{restatable}{theorem}{nonexistenceobs}\label{th:non-existence-comparison}
There is a CQE instance, which is both RL and EL, 
admitting an optimal view, but no optimal obstruction.
Conversely, there exists an RL CQE instance that admits an optimal
obstruction, but no optimal view.
\end{restatable}

\subsection{Linear Datalog and QL}
\label{sec:obstruction for OWL 2 QL}

We now show how to apply resolution-based techniques to 
compute optimal obstructions for linear Datalog CQE instances and then adapt the results to QL.
In fact, we can guarantee not only existence of optimal obstructions for such instances, but also uniqueness and polynomiality 
of corresponding censors.

Our solution for linear Datalog instances is based on the 
computation of the set $\goals{\Inst}$ of existential closures of 
goals in the proofs of policies. However, since all the rules 
in the ontology are linear and the body of the policy is an atom (recall that the rule corresponding to the policy should be linear as well), each of these goals consists of a single atom, except the last goal in each proof, which is empty. There are only polynomial number of such atoms (up to renaming of variables). So, all the proofs can be represented by a single finite \emph{proof graph} with atoms and the empty conjunction (denoted by $\top$) as nodes, and SLD resolution steps as edges. This is illustrated by the following example.

\begin{example}
\label{ex:graph-proof}
Consider a CQE instance with ontology
$$
\set{\likes(x,y) {\rightarrow} \film(y), \likes(x,y) {\rightarrow} \movieFan(x)},
$$
dataset $\likes(\john, \movie)$,
and policy $\movieFan(\john)$.
A fragment of the proof graph  is given in Figure~\ref{fig:proof_graph}.
\end{example}

Using proof graphs we can compute optimal censors.

\begin{restatable}{theorem}{algorithmForQLIsCorrect}
\label{thm:optimal censors for linear RL}
Let $\Inst = (\O, \D, P)$ be a linear Datalog CQE instance, and
	let $S$ be the set of all nodes 
	in the proof graph of $\O \cup \D$ 
	on the paths from facts $P(\vec a)$ with any tuple of constants $\vec a$ to $\top$. 
Then, the Boolean UCQ
\begin{align*}
U =  \bigvee\nolimits_{G \in S \setminus \{\top\}} 
  	\exists^*G
\end{align*}
is an optimal obstruction 
computable in polynomial time, and  
 $\ocens{\Inst}{U}$ is the unique optimal censor for \Inst.
\end{restatable}

\begin{example}
For the instance in Example~\ref{ex:graph-proof}
there is only one path in the proof graph from the policy to $\top$, and 
	$S = 
	\{\movieFan(\john),\likes(\john,y), \top\}$.
Thus, $movieFan(\john) \lor \exists y.\likes(\john,y)$ 
is optimal.
\end{example}

Finally, note that the transformation of a QL ontology $\O$ to an RL ontology $\dat{\O}{\Const}$
	given in Definition~\ref{Xidef},
	preserves linearity of rules.
Hence, Proposition~\ref{thm:combined-approach} with Theorem~\ref{thm:optimal censors for linear RL}
	yield the following result. 
 
\begin{restatable}{theorem}{obstructionForQL}
\label{thm:obstruction censor for QL}
Every QL CQE instance admits
a unique optimal censor based 
on an obstruction that can be
computed in polynomial time.
\end{restatable}

\begin{figure}[t]
\begin{center}
\includegraphics[width=.47\textwidth]{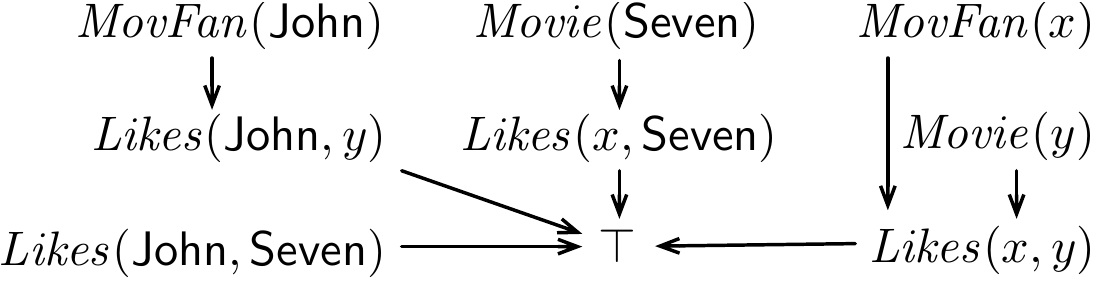}
\end{center}
\caption{Fragment of proof graph from Example~\ref{ex:graph-proof}}
\label{fig:proof_graph}
\end{figure}

\section{Related Work}
\label{sec: related work}

The formal study of privacy in databases has
received significant attention.
CQE for propositional databases with complete information 
has been studied in  
\cite{DBLP:journals/tods/SichermanJR83,DBLP:journals/tkde/BonattiKS95,DBLP:journals/dke/BiskupB01,DBLP:journals/ijisec/BiskupB04}. 
CQE was extended to (propositional) incomplete databases
in \cite{DBLP:journals/ijisec/BiskupW08}.  
\citeA{DBLP:journals/jcss/MiklauS07}
studied \emph{perfect privacy}.
Perfect privacy, however, is very strict
and may preclude publishing of any meaningful information when
extended to ontologies~\cite{DBLP:conf/ecai/GrauH08}.
View-based authorisation 
	was investigated in 
	\cite{DBLP:conf/sigmod/RizviMSR04,DBLP:conf/icdt/ZhangM05},
	and \citeA{DBLP:conf/icdt/DeutschP05} analysed 
	the implications to privacy derived 
	from publishing database views.

Privacy 
in the context of ontologies is a growing area of research.
Information hiding at the schema level was studied
in~\cite{DBLP:conf/ijcai/KonevWW09,DBLP:journals/jair/GrauM12}. 
Data privacy for $\mathcal{EL}$ and $\mathcal{ALC}$ DLs
	was investigated 
	in~\cite{DBLP:conf/ershov/StouppaS06,DBLP:conf/rr/TaoSH10}, 
	and
	the notion of a privacy-preserving reasoner
	was introduced in~\cite{DBLP:conf/webi/BaoSH07}.
	\citeA{DBLP:journals/jcss/CalvaneseGLR12}
	extended the view-based authorisation framework 
	by \citeA{DBLP:conf/icdt/ZhangM05} to DL ontologies.
	
An early work on non-propositional CQE is~\cite{Biskup:2007bh}.
CQE for ontologies has  been 
studied in~\cite{DBLP:conf/semweb/GrauKKZ13,DBLP:conf/semweb/BonattiS13,DBLP:journals/tdp/StuderW14}.
We extend \citeA{DBLP:conf/semweb/GrauKKZ13} with a wide range of new results:
\begin{inparaenum}[\it(i)]
\item we  
consider arbitrary CQs as policies rather than just ground facts;
\item we introduce obstruction censors, compare
their expressive power with that of view censors, 
characterise their optimality, 
and show how to compute obstructions for linear Datalog and QL ontologies;
\item we show undecidability of checking existence of an
optimal view censor and provide algorithms for guarded Datalog and all the OWL 2 profiles.
\end{inparaenum}	
We see our work as complementary to
\citeA{DBLP:conf/semweb/BonattiS13} and \citeA{DBLP:journals/tdp/StuderW14}.
The former 
focuses on situations where attackers have access to  
external sources of background knowledge;  
they identify additional vulnerabilities and 
propose solutions within the CQE framework.
The latter focuses on meta-properties
of general censors that, in contrast to ours, can also provide unsound 
answers or refuse queries. 

\section{Conclusions}
We studied CQE
in the context of ontologies.
Our results provide insights on the
fundamental tradeoff between accessibility and confidentiality of information.
Moreover, they yield a flexible way 
for system designers  
to ensure selective access to data.
In particular,
we proposed tractable 
 view based solutions 
for CQE instances with tree-shaped and linear Datalog and QL ontologies,
and tractable obstruction based solutions for linear Datalog and QL ontologies.
Our solutions can be implemented using off-the-shelf query answering infrastructure and 
provide a starting point 
for CQE system development.

\bibliographystyle{named}
 \clearpage
 \onecolumn
 \appendix 
\section{Appendix (Proofs)}

\subsection{Proofs for Section~\ref{sec:censor-types}}

\newenvironment{proofofclaim}{\noindent{\itshape Proof of claim.\ }}{\hfill $\square$\smallskip}

\newenvironment{claim}{\addtocounter{theorem}{1}\par\smallskip\noindent {\itshape Claim~\thetheorem.}}{\par\smallskip}

Before proving Theorem~\ref{th:simulation} we present the following notation and a lemma.
Let $\str{I}$ be a finite structure and 
$f$ a function associating a fresh variable to each 
domain element of $\str{I}$.
The query $\query{\str{I}}$ for $\str{I}$ is the Boolean CQ defined as follows, with
$R_1, \ldots R_n$ the predicates interpreted by $\str{I}$:
\begin{align*}
\query{\str{I}} = \exists^* \bigwedge_{1 \leq i \leq n} \{R_i(f(u_1),\ldots, f(u_{m_i})) \mid (u_1,\ldots,u_{m_i}) \in R_i^{\str{I}} \}.
\end{align*}  

Given a BCQ $Q$, 
	denote $\freeze{Q}$ the structure interpreting each $R$, occurring in $Q$,
	with $(f(u_1), \ldots f(u_n))$ for every atom $R(u_1, \ldots, u_n)$ in $Q$, 
	where $f$ maps each constant in $Q$ to itself and each variable $y$ to a fresh constant $d_{y}$.

\begin{lemma}
\label{lem:known-CSP}
Let $\str{J}$ be a finite structure and let
$\mathbb{C}$ be a class of finite structures. Then, 
the following holds:
\begin{align*}
	\set{\I \in \mathbb{C} \mid \I \not\homo \J}
	= \set{ \str{I} \in \mathbb{C} \mid \str{I} \models \bigvee_{\str{K} 
		\in \mathbb{C}, \str{K} \not\homo \str{J}} \query{\str{K}} }.
\end{align*}
 \end{lemma}

\begin{proof}
Let $\str{I} \in \mathbb{C}$ be such that $\str{I} \not\homo \str{J}$; 
	clearly, $\str{I} \models \query{\str{I}}$ 
	and hence $\str{I} \models \bigvee_{\str{K} \in \mathbb{C}, \str{K} \not\homo \str{J}} 
	\query{\str{K}}$, as required.
Conversely, assume that $\str{I} \in \mathbb{C}$ 
	is such that $\str{I} \models \bigvee_{\str{K} \in \mathbb{C}, \str{K} \not\homo \str{J}} 
	\query{\str{K}}$; 
	then, there exists $\str{K}$ such that $\str{K} \in \mathbb{C}$, $\str{K} \not\homo \str{J}$ 
	and $\str{I} \models \query{\str{K}}$. 
The latter implies that $\str{K} \homo \str{I}$
	and hence we can deduce $\str{I} \not\homo \str{J}$, as required 
	(otherwise, we would have by composition of homomorphisms that $\str{K} \homo \str{J}$, 
	which is a contradiction).
\end{proof}

\censorsimulation*
\begin{proof}\mbox{}\\
$\boldsymbol{(\Leftarrow)}$
Assume that $U$ defines 
	$\mathbb{C} \setminus \set{\I \in \mathbb{C} \mid \I  \homo \chase{\O}{\V}}$,
	which is equal to $\set{\I \in \mathbb{C} \mid \I \not\homo \chase{\O}{\V}}$.
Then, for each $\str{I} \in \mathbb{C}$ we have that $\str{I} \not\homo \chase{\O}{\View}$ 
	iff $\str{I} \models U$. 
By Lemma~\ref{lem:known-CSP},
	the following holds for each $\str{I} \in \mathbb{C}$:
\begin{align}
\label{eq:th-CSP-condition}
\str{I} \models U \quad \text{ iff } \quad \str{I} \models \bigvee_{\str{K} \in \mathbb{C}, \str{K} \not\homo \chase{\O}{\View}} \query{\str{K}}. 
\end{align}
Let $Q(\vec x)$ be a CQ, and let $\vec t 
\in \cert{Q}{\O}{\D}$, which implies that
$\freeze{Q(\vec t)} \homo \chase{\O}{\D}$ and hence $\freeze{Q(\vec t)} \in \mathbb{C}$.
We show that $\vec t
\in \vcens{\Inst}{\View}(Q)$ iff $\vec t \in \ocens{\Inst}{U}(Q)$.

For the forward direction, 
	assume that $\vec t \in  \vcens{\Inst}{\View}(Q)$; 
	then, $\O \cup \View \models Q(\vec t)$ 
	and hence $\freeze{Q(\vec t)} \homo \chase{\O}{\View}$. 
We can then conclude $\freeze{Q(\vec t)} \not\models \bigvee_{\str{K} \in \mathbb{C}, 
	\str{K} \not\homo \chase{\O}{\View}} \query{\str{K}}$ 
	(otherwise, $\str{K} \homo \freeze{Q(\vec t)}$ 
	for some $\query{\str{K}}$ in $U$ 
	and since we have established 
	that $\freeze{Q(\vec t)} \homo \chase{\O}{\View}$ 
	and homomorphism compose we would have $\str{K} \homo \chase{\O}{\View}$ 
	which is a contradiction). 
But then,
	Equation~\eqref{eq:th-CSP-condition} implies that $\freeze{Q(\vec t)} \not\models U$
	and by the definition of obstruction-censor that $\vec t \in \ocens{\Inst}{U}(Q)$, 
	as required.

For the backward direction, 
	assume now that $\vec t \in \ocens{\Inst}{U}(Q)$. 
Then, by the definition of obstruction censor we have $\freeze{Q(\vec t)} \not\models U$. 
By Equation~\eqref{eq:th-CSP-condition} 
	we then have $\freeze{Q(\vec t)} \not\models \bigvee_{\str{K} \in \mathbb{C}, 
	\str{K} \not\homo \chase{\O}{\View}} \query{\str{K}}$.
Lemma~\ref{lem:known-CSP} immediately implies 
	that $\freeze{Q(\vec t)} \not\in \set{\I \in \mathbb{C} \mid \I \not\homo \chase{\O}{\V}}$. 
From this, we must conclude
	that $\freeze{Q(\vec t)} \in \set{\I \in \mathbb{C} \mid \I  \homo \chase{\O}{\V}}$ 
	and hence $\freeze{Q(\vec t)} \homo \chase{\O}{\View}$, 
	which implies $\O \cup \View \models Q(\vec t)$ 
	and $\vec t \in  \vcens{\Inst}{\View}(Q)$, as required.

\medskip
\noindent
$\boldsymbol{(\Rightarrow)}$ 
Assume that $\ocens{\Inst}{U} = \vcens{\Inst}{\View}$. 
To show that $U$ defines 
	$\set{\I \in \mathbb{C} \mid \I  \not\homo \chase{\O}{\V}}$,
	we prove that $\str{I} \models U$ iff $\str{I} \not\homo \chase{\O}{\View}$
	for every structure $\str{I}$ in $\mathbb{C}$. %
If $\str{I} \homo \chase{\O}{\D}$ and $\str{I} \models U$, 
	then $\ocens{\Inst}{U}(Q^{\str{I}}) = \false$. 
Since $\ocens{\Inst}{U} = \vcens{\Inst}{\View}$,
	we also have that $\vcens{\Inst}{\View}(Q^{\str{I}}) = \false$ 
	and hence $\O \cup \View \not\models Q^{\str{I}}$. 
Consequently,
	$\str{I} \not\homo \chase{\O}{\View}$, as required.
If $\str{I} \not\homo \chase{\O}{\View}$, 
	then $\O \cup \View \not\models Q^{\str{I}}$; 
	consequently,
	$\vcens{\Inst}{\View}(Q^{\str{I}}) = \false$. 
Since $\ocens{\Inst}{U} = \vcens{\Inst}{\View}$,
	we have $\ocens{\Inst}{U}(Q^{\str{I}}) = \false$ 
	and hence, since $\str{I} \homo \chase{\O}{\D}$, we necessarily have $\str{I} \models U$.
\end{proof} %
\censorNotSimulationGeneral*

\begin{proof}
First we illustrate that obstruction censors cannot always simulate
view censors. Consider CQE instance
$\Inst = (\emptyset,\D,\emptyset)$,
where $\D$ represents an undirected graph with nodes ``green'' $g$ and ``blue'' $b$, which
are connected by $edge$ in all possible ways:
\begin{equation*}
\D = \{ \edge(g,b), \edge(b,g), \edge(b,b), \edge(g,g) \}.
\end{equation*} 
Clearly, $\D$ entails every Boolean CQ over the $\edge$ relation 
and thus every graph can be homomorphically embedded into $\D$.
Consider $\View = \{\edge(g,b), \edge(b,g)\}$.
Since the ontology is empty, $\chase{\emptyset}{\View} = \View$ and
$\set{\I \mid \I \text{ is finite, } \I \homo \chase{\O}{\D}, 
	\text{ and } \I \not\homo \V}$
is the class of all graphs that are not 2-colourable.
It is well-known that this class of graphs is not first-order definable and hence
cannot be captured by a UCQ.

\medskip
Next we construct an obstruction censor which cannot be simulated by a
view censor. Consider the
instance $\Inst = (\O,\D,\emptyset)$, where
$\D =   \{\edge(a,a)\}$ and $\O$ consists of the single transitivity rule
\begin{align*}
\edge(x,y) \wedge \edge(y,z) \rightarrow \edge(x,z). 
\end{align*}
Clearly, $\O \cup \D$ entails each Boolean CQ over
the $\edge$ relation. 
Consider  obstruction $U  =  \exists y.\edge(y,y)$, which
defines the class of directed graphs with self loops.
Suppose that some view $\View$
realises $\ocens{\Inst}{U}$. 
By Theorem \ref{th:simulation}, the obstruction
$U$ must define 
$\set{\I \in \mathbb{C} \mid \I \not\homo \chase{\O}{\View}}$
where
$\mathcal{C}$ is the class of all directed graphs.
Thus, any graph $G$ must satisfy the property
\begin{align*}
G \text{ has no self loops iff } G \homo \chase{\O}{\View}.
\end{align*}
Due to the rule in $\O$, we conclude that
$\View$ is a DAG, that is, it has no $\edge$-loops.
Take a DAG $G$ extending
	(a graph isomorphic to) $\chase{\O}{\View}$
	with a new node $v$
	and 
	edges connecting all its sink nodes to $v$.
Clearly $G$ has no self loops, but $G \not\homo \chase{\O}{\View}$, which is a contradiction.\looseness=-1
\end{proof} 
\subsection{Proofs for Section~\ref{sec:view-censors}}

\undecTheorem*

\begin{proof}

The proof is by reduction from the following problem: does a deterministic Turing machine without a final state have a repeated configuration? This problem is undecidable by Rice's Theorem.

Formally, for every such Turing machine $M = (\Gamma, \mathcal{Q}, q_0, \delta)$ with $\Gamma$ a tape alphabet, which include the blank symbol $0$, $\mathcal{Q}$ a set of states, $q_0 \in \mathcal Q$ an initial state, and $\delta: \Gamma \times \mathcal Q \to \Gamma \times (\mathcal Q \setminus \{q_0\})\times \{+, -\}$ a transition function, we construct an Datalog CQE instance $\Inst_M = (\O, \D, P)$ such that it admits an optimal view if and only if $M$ starting on the empty tape has a repeated configuration. The notion of configuration is as usual---it is the content of the tape and the head pointer to a cell on the tape. Note that the transition function $\delta$ is defined is such a way that the initial state does not appear in a computation anywhere except the initial configuration. This clearly does not affect the undecidability of the problem. We also assume, that the tape of the machine is infinite in both directions, and all of it can freely be used for computations.

We start the construction of $\Inst_M$ from the dataset $\D$. It uses only one constant $a$ and consists of three binary atoms
$$
R(a, a), S(a, a), T(a, a).
$$
The predicate $T$ is intended to point to the next cell on the tape, the predicate $S$ points to the same cell in the following configuration, and the predicate $R$ is responsible for initialisation.
We start the definition of the ontology $\O$ with the description of the role of $R$.
Let $\O$ contain rules
\begin{eqnarray}
R(x,x)  \rightarrow I(x), \label{in1}\\
R(x,y) \wedge R(y, z) \rightarrow R(x, z), \label{in2} \\
R(x, y) \wedge I(y) \rightarrow I(x). \label{in3}
\end{eqnarray}
As we will see formally later, these rules guarantee that if  $\Inst_M$ admits an optimal view, then this view contains the fact $I(a)$.
This fact initialises the tape by means of the following rules (conjunction in heads is just a syntactic sugar):
\begin{eqnarray}
I(x)  \rightarrow I^+(x) \land I^-(x) \land C_{q_0} (x)  \land A_0 (x), \label{in4} \\
I^+ (x) \land T(x, y) \rightarrow I^+(y) \land C_\emptyset (y) \land A_0(y), \label{in5} \\
I^- (y) \land T(x, y) \rightarrow I^-(x) \land C_\emptyset (x) \land A_0(x). \label{in6}
\end{eqnarray}
In these rules $C_{q_0}$ is a unary predicate indicating that the head is pointing to the first cell and the state is $q_0$. For each other state $q$ in $\mathcal Q$ the vocabulary contains the corresponding predicate $C_q$. The rest of the tape should always be marked by predicates $C_\emptyset$ indicating that the head does not point to this cell. Similarly, if in some configuration a cell contains an alphabet symbol $g \in \Gamma$, then this is indicated by the predicate $A_g$; for example, the rules above ensure that the tape is initialised by the symbol $0$. To ensure the consistency of the computation grid, constructed by means of tape and time predicates $T$ and $S$, the ontology contains the rules
\begin{eqnarray}
T(x, y)  \land T(z, u) \land S(y, u) \rightarrow S(x, z), \label{grid1}\\
T(x, y)  \land T(z, u) \land S(x, z) \rightarrow S(y, u). \label{grid2}
\end{eqnarray}
Finally, we need to make the adjacent configurations consistent. In particular, the content of each cell, as well as the fact that the head is pointing to this cell in some particular state, that is, the cell's $C_q$ and $A_g$ labels, is completely defined by the labels of the three cells in the previous configuration. So, abbreviating $T(x, y)  \land T(y, z) \land S(y, u) \land A_{g^-}(x) \land A_g(y) \land A_{g^+}(z)$ by $\phi(x, y, z, u)$, the ontology $\O$ contains the rules
$$
\begin{array}{ll}
\phi(x, y, z, u) {\land} C_\emptyset (x) {\land} C_q (y) {\land} C_\emptyset (z) \rightarrow C_\emptyset (u) {\land} A_{g'} (u),& \text{for all } g^-, g^+ \in \Gamma, \text{ if } \delta(g, q) = (g', q', d) \text{ for some } q', d, \\
\phi(x, y, z, u) {\land} C_q (x) {\land} C_\emptyset (y) {\land} C_\emptyset (z) \rightarrow C_\emptyset (u) {\land} A_{g} (u), & \text{for all } g, g^+ \in \Gamma, \text{ if } \delta(g^-, q) = (g', q', -) \text{ for some } g', q', \\
\phi(x, y, z, u) {\land} C_\emptyset (x) {\land} C_\emptyset (y) {\land} C_q (z) \rightarrow C_\emptyset (u) {\land} A_{g} (u), & \text{for all } g, g^- \in \Gamma, \text{ if } \delta(g^+, q) = (g', q', +) \text{ for some } g', q', \\
\phi(x, y, z, u) {\land} C_q (x) {\land} C_\emptyset (y) {\land} C_\emptyset (z) \rightarrow C_{q'} (u) {\land} A_{g} (u), & \text{for all } g, g^+ \in \Gamma, \text{ if } \delta(g^-, q) = (g', q', +) \text{ for some } g', \\
\phi(x, y, z, u) {\land} C_\emptyset (x) {\land} C_\emptyset (y) {\land} C_q (z) \rightarrow C_{q'} (u) {\land} A_{g} (u), & \text{for all } g, g^- \in \Gamma, \text{ if } \delta(g^+, q) = (g', q', -) \text{ for some } g', \\
\phi(x, y, z, u) {\land} C_\emptyset (x) {\land} C_\emptyset (y) {\land} C_\emptyset (z) \rightarrow C_\emptyset (u) {\land} A_{g} (u), & \text{for all } g^-, g, g^+ \in \Gamma.
\end{array}
$$

Having the ontology defined, we complete the construction with specifying the policy. It consists of several BCQs, but the translation to a single CQ by means of several rules in the ontology is straightforward. The policy $P$ guarantees that a cell cannot contain several alphabet symbols, the machine cannot be in several states, and the head cannot simultaneously point and not point to a cell. This is formalised as the following set of BCQs:
$$
\begin{array}{ll}
\exists x.\, A_{g}(x) \land A_{g'}(x), & \text{for all } g, g' \in \Gamma \text{ such that } g \neq g', \\
\exists x.\, C_{q}(x) \land C_{q'}(x), & \text{for all } q, q' \in \mathcal Q \cup \{\emptyset\} \text{ such that } q \neq q'.
\end{array}
$$

Completed the construction, next we formally prove that $M$ has a repeated configuration if and only if $\Inst_M$ has a (finite) optimal view. We start with forward direction.

\medskip

\noindent ($\Rightarrow$)

Let the first pair of repeated configurations of $M$ have numbers $m$ and $n$, while the smallest (non-positive) number of a cell whose content was changed during the computation is $k+1$, and the biggest (non-negative) such number is $\ell-1$ (we assume that initially the head is pointing to the cell number 0). Note that $k$ and $\ell$ are finite, because a computation cannot use infinite number of cells in finite number of steps. In fact, $k > -n$ and $\ell < n$.

The view $\View$ makes use of constants $a_{ij}$ with $-1 \leq i < n$ and $k \leq j \leq \ell$, such that $a_{00} = a$ and all others are anonymous copies of $a$. By means of binary predicates $S$ and $T$ these constants form a grid, that is the view contains atoms
$$
\begin{array}{ll}
S(a_{(i-1)j}, a_{ij}), & \text{for all } 0 \leq i < n, k \leq j \leq \ell, \\
T(a_{i(j-1)}, a_{ij}), & \text{for all } {-}1 \leq i < n, k < j \leq \ell.
\end{array}
$$
The grid is ``folded'' on all the sides, in the configuration number $i = {-}1$ and cells number $k$ and $\ell$ by means of self loops, and on repeated configurations $m$ and $n$:
$$
\begin{array}{ll}
S(a_{({-}1)j}, a_{({-}1)j}), & \text{for all } k \leq j \leq \ell, \\
T(a_{ik}, a_{ik}), & \text{for all } {-}1 \leq i < n, \\
T(a_{i\ell}, a_{i\ell}), & \text{for all } {-}1 \leq i < n, \\
S(a_{(n-1)j}, a_{mj}), & \text{for all } k \leq j \leq \ell.
\end{array}
$$
Each configuration with number $0 \leq i < n$ with the word $g_k \ldots g_\ell$ written on the part of the tape with cell numbers from $k$ to $\ell$, the state $q$, and the head pointing to the cell number $h$ is represented by means of the following facts:
$$
\begin{array}{ll}
A_{g_j}(a_{ij}), & \text{for all } k \leq j \leq \ell, \\
C_q(a_{ih}), \\
C_\emptyset(a_{ij}), & \text{for all } k \leq j \leq \ell, j \neq h.
\end{array}
$$
The auxiliary ``configuration'' number $-1$ is the same as a usual configuration with the empty tape, except that the head does not point anywhere:
$$
\begin{array}{ll}
A_{0}(a_{({-1)}j}), & \text{for all } k \leq j \leq \ell, \\
C_\emptyset(a_{({-1)}j}), & \text{for all } k \leq j \leq \ell, j \neq h.
\end{array}
$$
The constant $a$ is in the initialisation predicates:
$$
R(a, a), I(a).
$$
Finally, each configuration with number $-1 \leq i < n$ (i.e., including the auxiliary one) has cells with numbers $k'$ and $\ell'$ such that all the cells between $k$ and $k'$, as well as all the cells between $\ell'$ and $\ell$ contain $0$ and do not have the head pointing on them. the first group is marked by $I^-$ and the second by $I^+$:
$$
\begin{array}{ll}
I^-(a_{(ij}), & \text{for all } k \leq j \leq k', \\
I^+(a_{(ij}), & \text{for all } \ell' \leq j \leq \ell.
\end{array}
$$

It is straightforward to see that $\View \models \O$ and $\O \cup \View \not  \models P$, that is, $\View$ is a confidentiality preserving view for $\Inst_M$. Also, it is a matter of technicality to check that the view is indeed optimal.

\bigskip

\noindent ($\Leftarrow$)

Next we show that if the machine $M$ does not have a repeated configuration, then there is no optimal view for the instance $\Inst_M$.
Assume for the sake of contradiction that such a view $\View$ exists. Without loss of generality we may assume that $\View \models \O$. The first fact we need is the following claim.

\medskip

\begin{claim}
The view $\View$ contains the atom $I(a)$.
\end{claim}
\begin{proof}
Whatever is the shape of $\View$, it entails the BCQs 
$$
Q^R_i = \exists x_1 \ldots \exists x_i.\, R(a, x_1) \wedge R(x_1, x_2) \wedge \dots \wedge R(x_{i-1}, x_i) \text{ for all } i \geq 1.
$$
Since $i$ is unbounded, but $\View$ is finite, there exists $i_0$ such that there is a homomorphism from the body of $Q^R_{i_0}$ to $\View$ which sends different $x_j$ and $x_k$ to the same constant. This means that there is an $R$-loop of some length in $\View$, which is connected by an $R$-chain from $a$. By the rules (\ref{in1})--(\ref{in3}) this implies that $I(a)$ is a fact in $\View$.
\end{proof}

Similarly to the proof of the claim above, whatever is the shape of $\View$, it entails the BCQs
Whatever is the shape of $\View$, it entails the BCQs 
$$
Q^S_i = \exists x_1 \ldots \exists x_i.\, S(a, x_1) \wedge S(x_1, x_2) \wedge \dots \wedge S(x_{i-1}, x_i) \text{ for all } i \geq 1.
$$
Since $\View$ is finite, this implies that there is the (finite) biggest number $n - 1$ such that the body of $Q^S_n$ has a homomorphism to $\View$ which sends different $x_j$ to different constants.

Consider now a ``grid'' BCQ $Q^{S, T}$ that consists of the following atoms:
$$
\begin{array}{ll}
S(x_{(i-1)j}, x_{ij}), & \text{for all } 0 < i \leq n, {-}n \leq j \leq n, \\
T(x_{i(j-1)}, x_{ij}), & \text{for all } 0 \leq i \leq n, {-}n < j \leq n, \\
x_{00} = a.
\end{array}
$$
This query is also ``harmless'', that is, should be entailed by $\View$ whatever is its shape.
Since this BCQ has a chain of $S$ starting from $a$ of length greater than $n-1$, for any homomorphism from the body of $Q^{S, T}$ to $\View$ there are numbers $k$ and $\ell$ such that this homomorphism sends $x_{k0}$ and $x_{\ell0}$ to the same constant. Let $h$ be such a homomorphism, and $k$, $\ell$ be the numbers corresponding to $h$. By rules (\ref{grid1}) and (\ref{grid2}) we have that $\View$ contains atoms
\begin{equation}
S(x_{(\ell-1)j}, x_{kj}), \quad \text{for all } {-}n \leq j \leq n. \label{loop}
\end{equation}

On the other hand, by the fact that $I(a)$ is in $\View$ and the rules (\ref{in4})--(\ref{in6}) we have that the constants $h(x_{0j})$ for ${-}n \leq j \leq n$ represent the part of the initial configuration on cells with numbers from ${-}n$ to $n$. Furthermore, by means of the rules corresponding to the transition function of the machine, the constants $h(x_{ij})$ form the part of the configuration with number $i$ for all $0 < i < \ell$. By the same rules and atoms (\ref{loop}) we conclude that the constants $h(x_{kj})$ represent not only the part of the configuration number $k$, but also the part of the configuration number $\ell$. If these parts are different, then this discloses the policy, so they are the same. But the rest of the configuration, that is the content of the tape beyond the cells with numbers from ${-}n$ to $n$, is also the same for the configurations, because they are just full of symbols $0$ (the head cannot reach this part of the tape because it is too far). So, we come to the fact that $M$ has a repeated computation, which contradicts the precondition.
\end{proof} %
\begin{restatable}{proposition}{basicview}\label{prop:basic-view-properties}
The censor $\vcens{\Inst}{\View}$ based on a view $\V$ 
	is confidentiality preserving for a CQE instance $\Inst = (\O, \D, P)$
	if and only if
	$\O \cup \View \not\models P(\vec s)$ for each $\vec s \in \cert{P}{\O}{\D}$.
Additionally, it is optimal 
	if and only if 
	for each CQ $Q(\vec x)$ and each $\vec t \in \cert{Q}{\O}{\D}$, 
	the fact that $\O \cup \View \cup \set{Q(\vec t)} \not\models P(\vec s)$ 
	for any $\vec s \in \cert{P}{\O}{\D}$ implies that $\vec t \in \cert{Q}{\O}{\View}$.
\end{restatable}
\begin{proof}
Assume that  $\O \cup \View \not\models P(\vec s)$ for each $\vec s \in \cert{P}{\O}{\D}$. 
Trivially, $\O \cup \View \models \Fcens{\vcens{\Inst}{\View}}$ 
	and hence we have
	$\O \cup \Fcens{\vcens{\Inst}{\View}} \not\models P(\vec s)$ 
	for each $\vec s \in \cert{P}{\O}{\D}$, as required.

Assume now that $\cens$ is confidentiality preserving, 
	in which case $\O \cup \Fcens{\vcens{\Inst}{\View}} \not\models P(\vec s)$ 
	for each $\vec s \in \cert{P}{\O}{\D}$. 
Next, assume for the sake of contradiction 
	that $\O \cup \View \models P(\vec s)$ 
	for some $\vec s \in \cert{P}{\O}{\D}$.
Since $\O \cup \D \models P(\vec s)$, 
	by the definition of policy we have that $\vcens{\Inst}{\View}(P(\vec s)) = \true$ 
	and thus $P(\vec s) \in \Fcens{\vcens{\Inst}{\View}}$; 
	therefore, $\O \cup \Fcens{\vcens{\Inst}{\View}} \models P(\vec s)$, which is a contradiction.

\medskip

We next focus on the optimality statement. 
Assume that $\O \cup \View \cup \set{Q(\vec t)} \not\models P(\vec s)$ 
	for any $\vec s \in \cert{P}{\O}{\D}$ implies that $\vec t \in \cert{Q}{\O}{\View}$,
	while $\vcens{\Inst}{\View}$ is not optimal. 
Then, there is a confidentiality preserving censor $\cens$ 
	that extends $\vcens{\Inst}{\View}$;
	this means that for some CQ $Q(\vec x)$ and, 
	$\vec t \in \cert{Q}{\O}{\D}$ 
	we have $\vec t \in \cens(Q)$, but $\vec t \notin \vcens{\Inst}{\View}(Q)$.
The fact that $\vec t \notin \vcens{\Inst}{\View}(Q)$ 
	and $\vec t \in \cert{Q}{\O}{\D}$ 
	implies that $\vec t \not\in \cert{Q}{\O}{\View}$.
Furthermore, the fact that $\cens$ is confidentiality-preserving
	implies that $\O \cup \Fcens{\cens} \cup \{Q(\vec t)\} \not\models P(\vec s)$
	for any $\vec s \in \cert{P}{\O}{\D}$.
But then, since $\cens$ extends $\vcens{\Inst}{\View}$, 
	we have that $\Fcens{\vcens{\Inst}{\View}} \subseteq \Fcens{\cens}$ 
	and hence $\O \cup \Fcens{\vcens{\Inst}{\View}} \cup \set{Q(\vec t)} \not\models P(\vec s)$, 
	and therefore $\vec t \in \vcens{\Inst}{\View}(Q)$,
	which is a contradiction.

Finally, assume that there exists  some CQ $Q(\vec x)$ 
	and $\vec t \in \cert{Q}{\O}{\D}$ 
	such that $\O \cup \View \cup \{Q(\vec t)\} \not\models P(\vec s)$
	for each $\vec s \in \cert{P}{\O}{\D}$, 
	but $\O \cup \View \not\models Q(\vec t)$.
Then, we can define a censor $\cens$ 
	that behaves exactly like $\vcens{\Inst}{\View}$, 
	with the exception of answering $Q(\vec t)$ positively. 
Thus, $\Fcens{\cens} = \Fcens{\vcens{\Inst}{\View}} \cup \{Q(\vec t)\}$.
But then, since $\O \cup \View \cup \{Q(\vec t)\} \not\models P(\vec s)$ 
	for each $\vec s \in \cert{P}{\O}{\D}$ 
	and $\O \cup \View \models \Fcens{\vcens{\Inst}{\View}}$ 
	we have that $\O \cup \Fcens{\cens} \not\models P(\vec s)$, 
	which implies that $\cens$ is confidentiality preserving 
	and $\Fcens{\vcens{\Inst}{\View}}$ is not optimal, as required.
\end{proof} %
We say that a rule is \emph{normalised}
	if it has at most two atoms in its body;
	an ontology is \emph{normalised}
	if it is a set of normalised rules.
Clearly, any guarded ontology can be normalised.

\begin{definition}
\label{def:predicates closed under O}
Let $\Sigma$ be a signature,
	\O an ontology over $\Sigma$,
	and a subset $S$ of $\Sigma$ is a set of unary predicates.
$S$ is \emph{closed under \O} if
\begin{inparaenum}[\it (i)]
\item $\O \cup \set{A(x) \mid A \in S} \models C(x)$ implies that $C \in S$ and
\item if $A$ does not occur in \O, then $A \in S$.
\end{inparaenum}
\end{definition}

\guardedTreeShapedDatalogView*

\begin{proof}
\mbox{}\\
\noindent\textbf{Guarded, tree-shaped CQE instance.}
Algorithm~\ref{alg:view guarded datalog full}
	presents a procedure that builds a view for a given CQE instance $\Inst = (\O, \D, P)$.
We are going to show that
	if $\Inst$ is tree-shaped and guarded,
	then the algorithm returns an optimal view for \Inst.
By its construction, the constructed dataset \View is safe,
	so it remains to prove its optimality. 
Due to Proposition~\ref{prop:basic-view-properties}, it suffices to show
	that for each CQ $Q$ and a tuple $\vec t$ 
	such that $\vec t \in \cert{Q}{\O}{\D}$:
\begin{align}
\label{eq:view optimality criterion for guarded datalog}
	\text{if } \O \cup \View \cup \freeze{Q(\vec t)} \not\models \freeze{P(\vec s)}
	\text{ for each } \vec s \in \cert{P}{\O}{\D},
	\text{ then } \vec t \in \cert{Q}{\O}{\View}.
\end{align}
Observe the following.
\begin{enumerate}[(O1)]
\item
W.l.g.\ we can assume that $\View \cap \chase{\O}{\D} \incl \freeze{Q(\vec t)}$.
\label{it:opt view th guarded datalog, maximised query}

\item
If $\H$ is as defined in Algorithm~\ref{alg:view guarded datalog full},
$\chase{\O}{\D} \incl \H$
	and $\H \setminus \chase{\O}{\D}$
	consists of unary atoms only over fresh predicates
	introduced into $\O_E$ at Line~1.
	
\item
No rule of $\O_E$ can be applied to \View.
\end{enumerate}

Assume that $Q(\vec t)$ satisfies the ``if''-clause of
	Equation~\eqref{eq:view optimality criterion for guarded datalog}.
Since by the assumption $\vec t \in \cert{Q}{\O}{\D}$,
	then there is a homomorphism 
	$h$ from $\chase{\O}{\freeze{Q(\vec t)}}$ into $\chase{\O}{\D}$.
It is easy to see that
\begin{align*}
	h: \chase{\O}{\freeze{Q(\vec t)}} \rightarrow \chase{\O}{\D}
	\text{ iff }
	h: \chase{\O_E}{\freeze{Q(\vec t)}} \rightarrow \chase{\O_E}{\D}.
\end{align*}
We are going to use the following notations.
\begin{itemize}
\item
Denote $\chase{\O_E}{\freeze{Q(\vec t)}}$ as $\B$.

\item
Let $\X$ be a dataset and $d$ an element occurring in $\X$.
Then we define the set $\conc{\X}{d}$ as $\set{A \mid A(d) \in \X}$.
\end{itemize}

We are going to show the existence of a homomorphism
	$g: \B \rightarrow \View$,
	which would prove that $Q$ satisfies the ``then''-clause
	of Equation~\eqref{eq:view optimality criterion for guarded datalog}.
Let $d_1, \ldots, d_m$ be all the fresh constants from $\freeze{Q(\vec t)}$,
	let $d$ be an element from $\freeze{Q(\vec t)}$, and
	let $h$ be a homomorphism from $\B$
	into $\H$.
We claim that there exists $g$ that satisfies the following properties:
\begin{compactenum}
\item If $d$ is from $\O \cup \D$,
	then $g(d) = d$.
\item Let $d = d_i$ and $h(d_i) = a$.
Then $g(d_i) = a'$ such that $a' \in \sigma_a$, $a' \neq a$ and
	$\conc{\S}{a'} = \conc{\B}{d}$,
	where $\sigma_a$ is a set of all ``copies'' of $a$
	introduced by the algorithm
	(for example, see sub-routines 
	in Algorithm~\ref{alg:view guarded datalog sub-routines 1}).
\end{compactenum}

\medskip
It remains to show that $g$ does indeed exist and map 
	$\B$ into $\View$.
To this end, we need to show that
\begin{compactenum}
	\item for each element $d$ from $\freeze{Q(\vec t)}$,
	there is an element $a'$ in $\View$ satisfying
	the second property of $g$, and
	\item for each binary atom $R(d_1, d_2) \in \freeze{Q(\vec t)}$,
	there exists a corresponding binary atom $R(g(d_1), g(d_2)) \in \View$.
\end{compactenum}

\medskip
The former requirement follows from the construction of \View.
The latter one requires that
\begin{align}
\label{eq:homomorphism is safe}
	\cert{P}{\O_E}{\View \cup g(\B)} = \eset.
\end{align}
Note that Equation~\ref{eq:view optimality criterion for guarded datalog}
	implies that $\cert{P}{\O_E}{\View \cup \B}= \eset$.
Also observe that \emph{no rule from $\O_E$ is applicable to $\View \cup \B$}.
Indeed,
	no rule is applicable to $\View$ nor to $\B$ by construction.
Assume that a rule $r$ is applicable to $\View \cup \B$.
If the body of $r$ contains one atom,
	then we immediately obtain a contradiction.
If the body of $r$ contains two atoms then
	there exist an atom $f_1 \in \View$ and an atom $f_2 \in \B$
	such that $f_1 \land f_2$ is an instantiation of the body of $r$.
Assume that the atom in the body of $r$ corresponding to $f_1$
	is a guard of the rule;
	then all constants occurring in $f_2$
	occur in $f_1$ too.
Since \B and \View share only ``active'' constants (i.e., the ones from \Inst),
	we have that $f_2 \in \View \cap \B$
	(due to Observation~\ref{it:opt view th guarded datalog, maximised query}),
	and thus $r$ is applicable to \View,
	which gives a contradiction.
	
\medskip
Assume that Equation~\ref{eq:homomorphism is safe} does not hold.
Hence,
	there is a rule $r \in \O$ applicable to $\View \cup g(\B)$.
Recall that $r$ is not applicable to $\View$.
We have the following cases depending on the shape of $r$.

\begin{enumerate}
\item $r$ is of the form $A(x) \rightarrow C(x)$,
	$A(x) \land B(x) \rightarrow C(x)$, or
	$A \land B(x) \rightarrow C(x)$.
Clearly, in this case $r$ is applicable to $g(\B)$.
It is easy to see that $r$ is then applicable to \B
	since $\conc{\B}{d} = \conc{\View}{g(d)}$ for every $d$ in $\B$,
	which contradicts the observation above.

\item $r$ is of the form $R(x,y) \rightarrow \mathit{Head}(\vec x)$
	or $A \land R(x,y) \rightarrow \mathit{Head}(\vec x)$,
	where $\mathit{Head}(\vec x)$ is of one of the following forms
	for some unary $C$ of binary $Q$ predicate:
	$C(x)$, $C(y)$, $Q(x,y)$, or $Q(y,x)$.
Here we obtain a contradiction similarly to the previous case.

\item $r$ is of the form $R(x,y) \land A(x) \rightarrow \mathit{Head}(\vec x)$.
There are three cases.
	\begin{enumerate}[(a)]
	\item There are $a$, $b$, and $d_i$ such that $R(a,b) \in \View$ and $A(d_i) \in \B$,
		where $g(d_i) = a$.
	Since $r$ is not applicable to \B,
		then for any element $c$ occurring in \B, it is the case
		that $R(d_i,c) \notin \B$.
	Thus, $\delta_R(d_i) \notin \B$ and consequently $\delta_R(a) \notin \View$.
	The latter statement contradicts the assumption that $R(a,b) \in \View$.

	\item There are $a$, $b'$, and $d_i$ such that $R(d_i,b') \in \B$ and $A(a) \in \View$,
		where $g(d_i) = a$.
	Since $r$ is not applicable to \B,
		then $A(d_i) \notin \B$ and thus $A(g(d_i)) \notin \View$.
	This contradicts that $A(a) \in \View$.
	
	\item there are $a$, $b$, $b'$, $d_i$, and $d_j$ such that
		$R(d_i,b') \in \B$, $A(d_j) \in \B$, $g(d_i) = g(d_j) = a$,
		and $g(b') = b$.
	Then we conclude that $\conc{\B}{d_i} = \conc{\B}{d_j}$
		and consequently $A(d_i) \in \B$.
	If $\mathit{Head}(d_i,b')$ is equal to $C(d_i)$ or $C(b')$
		for some unary predicate $C$,
		then $C \in \conc{\B}{d_i}$ or $C \in \conc{\B}{b'}$, respectively,
		and thus $C(g(d_i)) \in \View$ or $C(g(b')) \in \View$, respectively.
	If $\mathit{Head}(d_i,b')$ is equal to $Q(d_i,b')$ for some binary predicate $Q$,
		then $\delta_Q \in \conc{\B}{d_i}$ and $\rho_Q \in \conc{\B}{b'}$,
		and thus $\delta_Q(g(d_i))$ and $\rho_Q(g(b'))$ are in \View;
		therefore, \CheckRole sub-routine of the algorithm
		would return $\mathtt{True}$ on input $(Q(g(d_i),g(b')), \View)$,
		and thus $Q(g(d_i),g(b')) \in \View$.
	Anyway, the obtained contradictions conclude the case.
	\end{enumerate}

\item $r$ is of the form $R(x,y) \land A(y) \rightarrow \mathit{Head}(\vec x)$.
This case is analogous to the previous one.
\end{enumerate}

Finally,
	$g(f)$ should be in \View for each binary atom $f \in \B$,
	since
	\begin{inparaenum}[\it (i)]
	\item Equation~\eqref{eq:homomorphism is safe} holds and
	\item binary atoms that do not discover the policy
		were exhaustively added to \View.
	\end{inparaenum}

Regarding the size of the \View,
	if $a$ is a constant occurring in $\Inst$ 
	and $C$ a set of unary predicates $A$ such that $A(a) \in \H$,
	then the number of ``copies'' of $a$ added by the algorithm
	is equal to a number of subsets of $C$ closed under $\O_E$
	(see Algorithm~\ref{alg:view guarded datalog sub-routines 1}).
Clearly, this number is exponential in $|\O|$ and polynomial in $|\D|$
	(see Definition~\ref{def:predicates closed under O}).

\medskip
\noindent\textbf{Multi-linear, tree-shaped CQE instance.}
Let a DPI $\Inst = (\O, \D, P)$ be such that \O is multi-linear Datalog.
Let $\View$ be a dataset returned by Algorithm~\ref{alg:view guarded datalog full}.
For every constant $a$,
	the set $\sigma_a$ contains the constant $a_{\A^*}$ such that
	$\A^*$ is a maximal subset of $\set{A \mid A(a) \in \chase{\O}{\D}}$
	closed under $\O_E$.
It is easy to check that the number of such subsets is polynomial in the size of $\O$.
The set $\A^*$ is a maximal set of labels (i.e., unary predicates) 
	among all constants in $\sigma_a$,
	i.e., if $a' \in \sigma_a$, 
	then $\set{A \mid A(a') \in \View} \incl \A^*$ for some $\A^*$.
We will also denote as $a^*$ an element of $\sigma_a$
	such that $\conc{\View}{a^*} = \A^*$.

\medskip
Let $b$ be a constant from \Inst and let $a'$ be from $\sigma_a$
	such that $R(a,b)$ is in $\chase{\O_E}{\D}$.
Since \Inst is multi-linear,
	\O does not include rules with bodies of the form
	$R(x,y) \land A(x)$ and thus whatever unary atoms $a'$ participates in,
	they cannot affect the atoms $b$ participates in.
Hence we conclude that
\begin{inparaenum}[\it (i)]
\item if $R(a',b)$ is in \View for some $a' \in \sigma_a$ and $b$ from \Inst,
	then so is $R(a^*,b)$ for a corresponding element $a^*$ from $\sigma_a$;
\item if $R(a',b')$ is in \View for some $a' \in \sigma_a$ and $b' \in \sigma_b$,
	then so is $R(a^*,b^*)$ for corresponding elements $a^*$ and $b^*$ 
	from $\sigma_a$ and $\sigma_b$, respectively.
\end{inparaenum}
Let $\View^*$ be a subset of $\View$ which is based on
	constants $a$ from \Inst and their copies $a^*$.
Clearly if, for some CQ $Q(\vec x)$ ,
	$\vec a \in \cert{Q}{\O}{\D}$ and
	$\vec a \in \cert{Q}{\O}{\View}$, then
	$\vec a \in \cert{Q}{\O}{\View^*}$,
	which proves optimality of $\View^*$.

The polynomial size of $\View^*$ follows from the observation
	that the sub-routine \Anonymise introduces only linearly many copies of a constant $a$
	for each set of labels, including $\A^*$.

\medskip
\noindent\textbf{Linear, tree-shaped CQE instance.}
Finally, assume that \O is linear.
Then, there is the unique maximal subset $\View_0$ of $\chase{\O_E}{\D}$
	such that $\cert{P}{\O_E}{\View_0} = \eset$,
	which gives the uniqueness of \View.
\end{proof} %

\HornDatalog*
 
\begin{proof} 
 First we show the confidentiality preservation of the censor.
Since $\vcens{\Inst'}{\View'}$ is confidentiality-preserving,
	we have that $\O' \cup \View' \not\models P(\vec s)$
	for each $\vec s \in \cert{P}{\O}{\D}$. 
Since $\O'$ is Datalog, 
	it is clear that $\chase{\O'}{\View} = \chase{\O'}{\View'}$; 
	thus, $\O' \cup \View \not\models P(\vec s)$
	for each $\vec s \in \cert{P}{\O}{\D}$. 
But then, 
	since $P$ is tree-shaped and $\O'$ is a rewriting of $\O$ 
	we have $\O \cup \View \not\models P(\vec s)$ for each $\vec s \in \cert{P}{\O}{\D}$
	(see~\cite{DBLP:conf/aaai/StefanoniMH13}), 
	as required.

\medskip

Now we concentrate on the optimality of the view. Assume by contradiction that $\vcens{\Inst}{\View}$
is not optimal, then, by Proposition~\ref{prop:basic-view-properties},
there exists a BCQ $Q$ such that \emph{(i)} $\O \cup \D \models Q$; \emph{(ii)} 
$\O \cup \View \not\models Q$;  and
\emph{(iii)} $\O \cup \View \cup \{Q\} \not\models P(\vec s)$ for each $\vec s \in \cert{P}{\O}{\D}$.
Since $\O \cup \D \models Q$ and $\O' \models \O$ we have \emph{(iv)} $\O' \cup \D \models Q$.
Furthermore, condition \emph{(iii)} 
	implies that $\O \cup \View \cup \freeze{Q} \not\models P(\vec s)$
and since $P$ is tree-shaped and $\O'$ is a rewriting of $\O$ we have 
$\O' \cup \View \cup \freeze{Q} \not\models P(\vec s)$, which 
by the fact that $\View \models \View'$ then
also implies that \emph{(v)} $\O' \cup \View' \cup \{Q\} \not\models P(\vec s)$
for each $\vec s \in \cert{P}{\O}{\D}$. 
But then, \emph{(iv)} and \emph{(v)} and the fact that
$\View'$ is optimal for $\Inst'$ we must have $\O' \cup \View' \models Q$.
Since $\View = \chase{\O'}{\View'}$ we have $\View \models Q$, which
contradicts \emph{(ii)}.
\end{proof}

\subsection{Proofs for Section~\ref{sec:obstruction-optimal}}

For the sake of ease in the proofs for theorems and propositions of this section we will consider only the class of BCQs with constants.
Clearly, any results obtained for this class will also hold for the class of all CQs.
Before proceeding to the main proofs, we introduce few definitions and lemmas.

Let \O be a Datalog ontology and \D a dataset;
	let $\mathbb{Q}'$ be a possibly infinite set of queries
	such that $\O \cup \D \models Q$ for each $Q \in \mathbb{Q}'$.
Then a censor $\cens_{\mathbb{Q}'}$ is defined as follows:
\begin{align*}
	\cens_{\mathbb{Q}'}(Q) = \true&\ \quad \text{ iff } \quad \cert{Q}{\O}{\D} = \true \text{ and } 
		\freeze{Q} \not\models Q' \text{ for each } Q' \in \mathbb{Q}'.
\end{align*}

\begin{restatable}{lemma}{UCQCensorMinimalSet}
\label{prop:UCQ censor over minimal set}
Let $\Inst = (\O, \D, P)$ be a CQE instance;
	let $\Upsilon$ be a pseudo-obstruction based on a subset $\mathbb{S}$ of $\goals{\Inst}$.
Then, $\cens_{\Upsilon} = \cens_{\goals{\Inst} \setminus \mathbb{S}}$.
\end{restatable}

\begin{proof}
Let $Q$ be a CQ such that $\cert{Q}{\O}{\D} = \true$.

Assume that $\cens_{\goals{\Inst} \setminus \mathbb{S}}(Q) = \false$;
	this yields that $\freeze{Q} \models Q'$ for some $Q' \in \goals{\Inst} \setminus \mathbb{S}$.
Then there exists $Q'' \in \Upsilon$ such that $Q' \models Q''$ 
	and thus $\freeze{Q} \models Q''$,
	i.e., $\cens_{\Upsilon}(Q) = \false$.
	
Assume that $\cens_{\Upsilon}(Q) = \false$;
	this yields that $\freeze{Q} \models Q''$ for some $Q'' \in \Upsilon$.
Note that $Q'' \in \goals{\Inst} \setminus \mathbb{S}$ since $\Upsilon \incl \goals{\Inst} \setminus \mathbb{S}$
	and thus $\cens_{\goals{\Inst} \setminus \mathbb{S}}(Q) = \false$.
\end{proof}

The lemma above allows us to speak of obstruction censors in terms of 
	either $\Upsilon$ or $\goals{\Inst} \setminus \mathbb{S}$, whatever way is more convenient to show the required results. %
We are going to show now that a censor \cens is optimal for a given CQE instance $\Inst$
	iff
	there exists a maximal subset $\mathbb{S}$ of $\goals{\Inst}$ 
	such that $\cens = \cens_{\goals{\Inst} \setminus \mathbb{S}}$.
But first we need the following notion of a normalised proof.

\begin{definition}
Let \O be a Datalog ontology, \D a dataset, and $G_0$ a goal.
A proof $\pi$ of length $n$ of $G_0$ in $\O \cup \D$ is \emph{normalised}
	if there is $k \leq n$ such that 
	$r_i \in \O$ for each $i < k$ and $r_j \in \D$ for each $j \geq k$.
Moreover, the number $k$ is called the \emph{frontier} of $\pi$, denoted $\frontier{\pi}$.
\end{definition}

Intuitively, a normalised proof $\pi$ works as follows:
	first we rewrite the initial query $G_0$ over the ontology \O
	until we obtain the query $G_{\frontier{\pi}-1}$ that can be mapped into \D,
	and then we perform such a mapping applying $(r_i, \theta_i)$ with $i \geq \frontier{\pi}$.
Observe that for every $G_i$ with $i < \frontier{\pi}$ it holds that $\O \cup G_i \models G_0$.

We exploit the following known result about SLD resolution over
Datalog ontologies.
\begin{restatable}{lemma}{proofnormalisation}
\label{lem:normalisation of a proof}
Let \O be a Datalog ontology, let \D be a dataset, and let
$G_0$ be a goal such that $\O \cup \D \models G_0$.
Then there exists a normalised SLD proof $\pi$ of $G_0$ in $\O \cup \D$.
\end{restatable}

\begin{lemma}
\label{lem:on characterisation of optimal censors for datalog}
Let $\Inst = (\O, \D, P)$ be a CQE instance with \O a Datalog ontology 
	and \cens a censor for $\O$ and $\D$.
Then \cens is optimal for \Inst iff there exists a maximal subset $\mathbb{S}$ of $\goals{\Inst}$
	such that
\begin{inparaenum}[\it (i)]
	\item $\O \cup \mathbb{S} \not\models P(\vec s)$ for each $\vec s \in \cert{P}{\O}{\D}$ and
	\item $\cens = \cens_{\goals{\Inst} \setminus \mathbb{S}}$.
\end{inparaenum}
\end{lemma}

\begin{proof}
We start with the ``only if''-direction.
Let us assume that such maximal subset $\mathbb{S}$ exists.
We  show that $\cens_{\goals{\Inst} \setminus \mathbb{S}}$
is optimal. 

First, we show that $\cens_{\goals{\Inst} \setminus \mathbb{S}}$ is confidentiality preserving.
Assume the contrary; then, 
there is a (finite) subset $\mathbb{F}$ of $\Fcens{\cens_{\goals{\Inst} \setminus \mathbb{S}}}$ 
	such that $\O \cup \mathbb{F} \models P(\vec s)$
	for some $\vec s \in \cert{P}{\O}{\D}$.
This yields the existence of proof $\pi$ of $P(\vec s)$ in $\O \cup \freeze{\mathbb{F}}$, where
$\freeze{\mathbb{F}} = \bigcup_{Q \in \mathbb{F}} \freeze{Q}$.
Due to Lemma~\ref{lem:normalisation of a proof}, we can assume that $\pi$ is normalised
with frontier $k+1$. Let $G_k$ be the goal right before frontier
in $\pi$. Since $\pi$ is normalised, then $G_k$ is proved
by using only facts from $\freeze{\mathbb{F}}$. So, 
we can write $G_k$ as $G_k = B_1 \wedge \ldots \wedge B_m$, where
each $B_j$ is the conjunction of all atoms that are proved using
facts only from a particular $\freeze{Q_j}$. Obviously, the order
in which these $B_j$ are proved is irrelevant, so let us assume that
all $B_j$ have been proved except for $B_i$; since, the different
$B_j$ can share variables, the remaining goal to prove may not be just $B_i$, but
rather $B_i \theta_i$, with $\theta_i$ some substitution. 
We make the following observations:
\begin{enumerate}
  \item $B_i \theta_i$ does not mention any constants not in $\O \cup \D$. Indeed, for any distinct
  queries $Q_k$, $Q_j$ in $\mathbb{F}$ we have that $\freeze{Q_k}$ and 
  $\freeze{Q_j}$ only share constants from $\O \cup \D$ $\freeze{Q_k}$; thus, if $B_i \theta_i$ contains some
  constant coming from $\freeze{Q_j}$ with $j \neq i$, it would not be possible to prove $B_i\theta_i$ using only facts from
  $\freeze{Q_i}$.
  \item There exists a proof of $P(\vec s)$ in $\O \cup \D$ 
  	such that $B_i \theta_i$ occurs as a subgoal.
  We construct such proof as follows. First, we can ``reach'' goal $G_k$ because it only requires
  rules from $\O$. Note also that each $B_j$ follows from $\O \cup \D$, so we can continue the proof
  by showing all $B_j$ except for $B_i$. Then, we can do it in such a way we reach precisely $B_i \theta_i$ as a subgoal.  
 \item $Q_i \models \exists^* B_i \theta_i$ since $B_i \theta_i$ is provable from  $\freeze{Q_i}$. 
\end{enumerate}
Observation 2 means that $B_i \theta_i \in \goals{\Inst}$ for all $1 \leq i \leq m$. Furthermore, since
the censor answers $\true$ for each $Q_i$ we have that $B_i \theta_i \in \mathbb{S}$. But then,
$\O \cup \mathbb{S} \models P(\vec s)$, which is a contradiction.

\medskip
Now we show the optimality of $\cens_{\goals{\Inst} \setminus \mathbb{S}}$. 
Clearly, a censor $\cens$ for $\Inst = (\O,\D,P)$ is optimal 
	if and only if for each CQ $Q(\vec x)$
and each $\vec{t} \in \cert{Q}{\O}{\D}$ the fact that
$\O \cup \Fcens{\cens} \cup \{Q(\vec t)\} \not\models P(\vec s)$ holds
for each $\vec s \in \cert{P}{\O}{\D}$ implies that $\O \cup \Fcens{\cens} \models Q(\vec t)$.
Due to this, $\cens_{\goals{\Inst} \setminus \mathbb{S}}$ is optimal if and only if
	for each $Q$ such that $\cert{Q}{\O}{\D} = \true$ and $\O \cup \Fcens{\cens_{\goals{\Inst} \setminus \mathbb{S}}} \cup \set{Q} \not\models P(\vec s)$,
	it holds that $\O \cup \Fcens{\cens_{\goals{\Inst} \setminus \mathbb{S}}} \models Q$.
Assume to the contrary that there exists a CQ $Q$ such that 
	$\cert{Q}{\O}{\D} = \true$ and $\O \cup \Fcens{\cens_{\goals{\Inst} \setminus \mathbb{S}}} \cup \set{Q} \not\models P(\vec s)$,
	but $\O \cup \Fcens{\cens_{\goals{\Inst} \setminus \mathbb{S}}} \not\models Q$.
The latter means that $\cens_{\goals{\Inst} \setminus \mathbb{S}}(Q) = \false$, that is, $\freeze{Q} \models Q'$,
	for some $Q' \in \goals{\Inst} \setminus \mathbb{S}$.
Recall that for any $Q \in \goals{\Inst} \setminus \mathbb{S}$ it holds that
	$\O \cup \mathbb{S} \cup \set{Q} \models P(\vec s)$ due to maximality of $\mathbb{S}$.
Observe that $\mathbb{S} \incl \Fcens{\cens_{\goals{\Inst} \setminus \mathbb{S}}}$;
	this yields $\O \cup \Fcens{\cens_{\goals{\Inst} \setminus \mathbb{S}}} \cup \set{Q} 
	\models P(\vec s)$,
	which contradicts the initial assumption and concludes the ``only if''-direction.

\medskip
Now we consider the ``if''-direction.
Let us now assume that $\cens$ is optimal, and let
$\mathbb{Q}' = \set{Q \mid \cens(Q) = \false}$.
Consider the following subset $\mathbb{S}$ of $\goals{\Inst}$: 
	$\mathbb{S} = \goals{\Inst} \setminus \mathbb{Q}'$.
To prove the ``if''-direction, it suffices to prove the following two conditions:
\begin{inparaenum}[\it (i)]
	\item
$\mathbb{S}$ is a maximal subset of $\goals{\Inst}$ 
	such that $\O \cup \mathbb{S} \not\models P(\vec s)$
	for each $\vec s \in \cert{P}{\O}{\D}$ and
	\item
$\cens_{\goals{\Inst} \setminus \mathbb{S}} = \cens$.
\end{inparaenum}

To show \emph{(i)}, 
	assume that $\O \cup \mathbb{S} \cup \set{Q} \models P(\vec s)$ 
	for some $\vec s \in \cert{P}{\O}{\D}$
	and some $Q \in \goals{\Inst}$.
Clearly, since by construction $\mathbb{S} \incl \Fcens{\cens}$, 
it holds that $\O \cup \Fcens{\cens} \cup \set{Q} \models P(\vec s)$,
and therefore $\cens(Q) = \false$, i.e. $Q \in \mathbb{Q}'$, which implies \emph{(i)}.

To show \emph{(ii)}, let us pick an arbitrary $Q$ such that $\O \cup \D \models Q$
	but $\cens(Q) = \false$ and hence $Q \in \mathbb{Q}'$. 
Since $\cens$ is optimal, 
	we have that $\O \cup \Fcens{\cens} \cup \{Q\} \models P(\vec s)$ for some
	$\vec s \in \cert{P}{\O}{\D}$, 
	so let $\mathbb{F}$ be any minimal subset of $\Fcens{\cens}$ 
	such that $\O \cup \mathbb{F} \cup \{Q\} \models P(\vec s)$.
Following the same arguments as we used in the ``only if'' direction
we have that there exists $G \in \goals{\Inst} \setminus \mathbb{S}$ 
such that $Q \models \exists^* G$; since $\exists^* G$ is part of the obstruction, 
	then $\cens_{\goals{\Inst} \setminus \mathbb{S}}(Q) = \false$.
Finally, assume that $\cens_{\goals{\Inst} \setminus \mathbb{S}}(Q) = \false$; then,
$Q \models \exists^* G$ for some $G \in \goals{\Inst} \setminus \mathbb{S}$. Since
$\goals{\Inst} \setminus \mathbb{S} \subseteq \mathbb{Q}$, we have that
$\cens(Q) = \false$, as required.
\end{proof} %

\optimalUCQcensor*

\begin{proof}

Let us prove Statement 1. Assume that $\Upsilon$ is a finite pseudo-obstruction.
By Lemma~\ref{prop:UCQ censor over minimal set}, we have that
$\cens_{\Upsilon} = \cens_{\goals{\Inst} \setminus \mathbb{S}}$.
By the ``only if'' statement in Lemma
\ref{lem:on characterisation of optimal censors for datalog}, we have that
$\cens_{\goals{\Inst} \setminus \mathbb{S}}$ is optimal. But then,
since $\Upsilon$ is finite, then $U$ is an obstruction.

\medskip

Next, we show Statement 2. Assume by contradiction that
each pseudo-obstruction is infinite, but there is an optimal
censor based on an obstruction $U$. Since $\ocens{\Inst}{U}$ is an 
optimal censor, then the ``if'' direction of Lemma
\ref{lem:on characterisation of optimal censors for datalog} tells us
that there exists a pseudo-obstruction $\Upsilon$ such that
$\ocens{\Inst}{U} = \cens_{\Upsilon}$. We can show that
then there exists a finite pseudo-obstruction
which contradicts the assumption above.
Pick any CQ $Q$ from $U$; then, clearly, $\ocens{\Inst}{U}(Q) = \false$ and hence
$\cens_{\Upsilon}(Q) = \false$. The latter implies that
there exists $Q'\in \Upsilon$ such that $Q \models Q'$. Let us now
construct $U' = \bigvee_{Q \in U} Q'$, which is finite and also
a ``subset'' of $\Upsilon$. To obtain a contradiction, it thus suffices
to show now that $\ocens{\Inst}{U'} = \cens_{\Upsilon}$.
Indeed, for each CQ $Q$ such that $\cert{Q}{\D}{\O} = \true$
	(recall that $\ocens{\Inst}{U} = \cens_{\Upsilon}$):
\begin{itemize}
	\item 
Assume that $\ocens{\Inst}{U}(Q) = \false$;
	then	 there is $Q'$ in $U$ such that $\freeze{Q} \models Q'$,
	which yields $\freeze{Q} \models Q''$ with $Q''$ from $U'$,
	and therefore $\ocens{\Inst}{U'}(Q) = \false$.
	
	\item
Assume that $\ocens{\Inst}{U'}(Q) = \false$;
	then $\freeze{Q} \models Q''$ for some $Q''$ in $U'$,
	and consequently, since $Q'' \in \goals{\Inst} \setminus \mathbb{S}$,
	we conclude that $\cens_\Upsilon(Q) = \false$.
\end{itemize}
The obtained contradiction concludes the proof.
\end{proof} %
\nonexistenceobs*

\begin{proof}

To show the first statement, consider $\Inst_1 = (\O_1,\D_1, P_1)$, where
$\D_1 = \set{R(a,a), A(a)}$, $P_1 = A(a)$, and the 
guarded RL (and EL) ontology
$\O_1 =$ $\set{R(x,y) \land A(y) \rightarrow A(x)}$. Since this CQE instance is 
guarded and tree-shaped, by Theorem~\ref{th:guarded} we can devise
an optimal view. 
No optimal obstruction, however, exists,
	which is shown in Example~\ref{example:no obstruction censor}.

To show the second statement, consider CQE instance $\Inst_2 = (\O_2,\D_2, P_2)$, 
	with $\D_2 = \set{R(a,a)}$,
	$P_2 = A(a)$, and $\O_2 = \{R(x_1,y) \land R(x_2, y) \rightarrow x_1 \approx x_2, R(x,y) 
	\rightarrow A(y)\}$.
From~\cite{DBLP:conf/semweb/GrauKKZ13} 
	we know that no optimal view exists for this instance, and the proof
	can be easily extended to our framework 
	(note that our notion of a censor $\vcens{\Inst}{\View}$ based on a view $\View$
	differs from the one in~\cite{DBLP:conf/semweb/GrauKKZ13} )
extends also to the case where views are not required to be sound. 
However, $U = A(a) \lor \exists x.\, R(x,a)$ is an optimal
obstruction, since there is only one
proof of $A(a)$ with subgoal $R(x,a)$.
\end{proof} 

\algorithmForQLIsCorrect*

\begin{proof}
Optimality and uniqueness follows from Theorem~\ref{th:characterisation}
 	and the facts that
 		\begin{inparaenum}[\it (i)]
 			\item the set $S$ is exactly $\goals{\Inst}$ 
 			\item the only maximal subset $\mathbb{S}$ of $\goals{\Inst}$ 
 			such that $\O \cup \mathbb{S}$ does not entail any $P(\vec s)$
 			is the empty set.			
 		\end{inparaenum}
To prove the former fact, first observe that any goal that can appear 
	in any SLD proof in $\O \cup \D$
	is isomorphic to one of the nodes of the proof-graph of $\O \cup \D$;
	then Fact~(i) follows directly from the construction of the proof-graph.
Fact~(ii) follows from the observation that each SLD proof in case of linear $\O$ is normalised,
	and therefore for each $Q \in S$ it holds that $\O \cup Q \models P(\vec s)$
	for some $\vec s \in \cert{\P}{\O}{\D}$.

Finally, polynomiality follows from 
 	the fact that in linear Datalog the size of
 	the proof-graph
 	is at most cubic in $|\O\cup\D|$.
\end{proof}

\obstructionForQL*

\begin{proof}
Let $\Inst = (\O, \D, P)$ be a CQE instance with $\O$ in QL.
Let $\cens'$ be the optimal censor for $\Inst' = (\dat{\O}{\Const}, \D, P)$,
	where $\Const$ is a set of constants of \Inst 
	and $\dat{\O}{\Const}$ is a linear Datalog ontology.
By Theorem~\ref{thm:optimal censors for linear RL}, $\cens' = \ocens{\Inst'}{U}$
	for the UCQ $U$ as defined in the theorem.
Let $\cens = \ocens{\Inst}{U}$.
We are going to show that \cens is an optimal censor for $\Inst$.

\medskip
\noindent
\textbf{Confidentiality preservation.}
Assume that \cens is not confidentiality preserving for $\Inst$,
	that is, $\O \cup \Fcens{\cens} \models P(\vec s)$ for some $\vec s \in \cert{P}{\O}{\D}$.
This means that there exist $Q_1, \ldots, Q_n \in \Fcens{\cens}$ such that
	$\O \cup \set{Q_1, \ldots, Q_n} \models P(\vec s)$;
	clearly, $\O \cup \D \models Q_i$ for each $i \in \set{1, \ldots, n}$.
By Proposition~\ref{thm:combined-approach}, %
	$\dat{\O}{\Const} \models \O$ and consequently 	
	$\dat{\O}{\Const} \cup \D \models Q_i$ for each $i \in \set{1, \ldots, n}$.
Since $\cens'$ is confidentiality preserving for $\Inst'$, we conclude that $\set{Q_1, \ldots, Q_n} \not\incl \Fcens{\cens'}$,
	so there is $j \in \set{1, \ldots, n}$ such that $\cens'(Q_j) = \false$;
	i.e., $\freeze{Q_i} \models U$.
The last entailment implies that $\cens(Q_j) = \false$,
	i.e., $Q_j \notin \Fcens{\cens}$, which yields a contradiction
	and thus $\cens$ is confidentiality preserving for \Inst.
	
\medskip
\noindent
\textbf{Optimality.}
Assume, for the sake of getting a contradiction, that \cens is not optimal for \Inst,
	that is, there exists $Q$ such that
	\begin{inparaenum}[\it (i)]
		\item $\O \cup \D \models Q$,
		\item $Q \notin \Fcens{\cens}$, and
		\item $\O \cup \Fcens{\cens} \cup \set{Q} \not\models P(\vec s)$ 
		for each $\vec s \in \cert{P}{\O}{\D}$.
	\end{inparaenum}
This yields $\freeze{Q} \models u$ for some disjunct $u$ in $U$ and consequently $\cens'(Q) = \false$.
Note that for each disjunct $u$ in $U$, 
	it holds that $\dat{\O}{\Const} \cup \set{u} \models P(\vec s)$ 
	for some $\vec s \in \cert{P}{\O}{\D}$;
	thus $\dat{\O}{\Const} \cup \set{Q} \models P(\vec s)$.
There are the following cases depending on the form of $u$.
\begin{itemize}
	\item If $u$ is of the form $A(a)$ or $R(a,b)$ with $a,b \in \Const$,
		then $\O \cup \set{u} \models P(\vec s)$ since, 
		due to Proposition~\ref{thm:combined-approach},
		$\dat{\O}{\Const}$ is a $\Const$-rewriting of \O;
		thus, $\O \cup \set{Q} \models P(\vec s)$ which yields a contradiction with \emph{(iii)}.
		
	\item If $u$ is of the form $\exists y. R(a,y)$ with $a \in \Const$,
		then let $\O_{\min}$ be a minimal subset of $\dat{\O}{\Const}$ such that
		$\O_{\min} \cup \set{u} \models P(\vec s)$.
	Due to the assumption, it holds $\O \cup \set{u} \not\models P(\vec s)$;
		thus, $\O_{\min} \not\incl \O$ and therefore $\O_{\min}$ includes one of the rules introduced by $\Xi$.
	That is, $\O_{\min}$ contains (some of) the following rules
		that come from the Skolemisation $\dat{r}{\Const}$ of some rule $r = A(x) \rightarrow \exists y.[S(x,y) \land B(y)]$ of Type~(3) in \O: %
	\begin{align}\label{eq:ql to datalog rules}
	A(x) \rightarrow P_S(x, c_{A,S}),\quad P_S(x,y) \rightarrow S(x,y),\quad \text{ and } P_S(x,y) \rightarrow B(y).
	\end{align}		
	Consider a proof $\pi = G_0 \rightarrow \ldots \rightarrow G_n$ of $P(\vec s)$ in $\dat{\O}{\Const} \cup \freeze{\exists y. R(a,y)}$,
		where $G_0 = P(\vec s)$. 
	Clearly, $G_i$ can be obtained from $G_{i-1}$ by applying a rule from $\O_{\min}$
	for each $i = 1, \ldots, n-1$, and $G_{n-1} = R(a,x')$ for some $x'$ since the last step of the proof is applying
		the only rule from $\freeze{\exists y. R(a,y)}$.
	Let $G_k$ be the first goal in $\pi$ obtained from $G_{k-1}$ by applying a rule from Equation~\eqref{eq:ql to datalog rules};
		clearly, $\O \cup \set{\exists^*G_{k-1}} \models \exists^* G_0$.
	We have the following cases.
	\begin{itemize}
		\item Assume that we apply the third rule from Equality~\eqref{eq:ql to datalog rules}
			to $G_{k-1} = B(b)$ for some constant $b$ 
			(note that a goal $B(x)$ with $x$ a Skolem constant cannot appear by applying QL rules except for Type $(3)$). %
		Then $G_k = P_S(x, b)$, and the only rule that has $P_S$ in its head is the first one from Equality~\eqref{eq:ql to datalog rules};
			however, this rule cannot be applied to $G_k$ since we cannot unify $b$ and $c_{A,S}$.
		Thus, this case is invalid.
		
		\item Assume that we apply the second rule from Equality~\eqref{eq:ql to datalog rules} to $G_{k-1} = S(b,d)$ for some constants $b$ and $d$.
		This case is always invalid due to the same reason as the previous one.
		
		\item Assume that we apply the third rule from Equality~\eqref{eq:ql to datalog rules} to $G_{k-1} = S(b,x)$ for some constant $b$ and Skolem constant $x$.
		Then, $G_k = P_S(b, x)$ and $G_{k+1}$ is obtained from $G_k$ by applying the first rule from Equation~\eqref{eq:ql to datalog rules};
			that is, $G_{k+1} = A(b)$.
		But then we have that $A(x) \rightarrow \exists y.[S(x,y) \land B(y)] \in \O$
			and consequently $\O \cup \set{A(b)} \models \exists^* G_{k-1}$.
		W.l.o.g. we can assume that starting from $G_{k+1}$ rules only from $\O$ are used,
			which means that $\O \cup \freeze{\exists y.R(s,y)} \models A(b)$.
		
		\item No other case is possible.
	\end{itemize}
	Thus $O \cup \set{u} \models P(\vec s)$ which contradicts \emph{(iii)}.
\end{itemize}
Thus, \cens is optimal for \Inst, which concludes the proof.
\end{proof} 

\section{Appendix (Algorithms)}
\label{app:algorithms}

\begin{algorithm2e}%
\SetKwInOut{Input}{INPUT}
\SetKwInOut{Output}{OUTPUT}
\SetKwFor{ForEach}{for each}{do}{end}
\SetKw{Create}{create}

\Input{a guarded CQE-instance $\Inst = (\O, \D, P)$}
\Output{a dataset $\View$}

\BlankLine

$\O_E := \O \cup \bigcup_{\text{binary } R \text{ in } \O}
	\set{R(x,y) \rightarrow \delta_R(x),\ R(x,y) \rightarrow \rho_R(y)}$;

$\H :=$ the minimal Herbrand model for $\O_E$ and \D;

$\View :=$ a maximal subset of unary atoms from \H
	s.t.\ $\cert{P}{\O_E}{\View} = \eset$;

\lForEach
{constant $a$ from \H}
{$\View := \Anonymise(a)$}

\lForEach
{$R(a,b) \in \H$ such that $R$ is not $\approx$}
{$\View := \AddRoles(R(a,b))$}

\Return \View;

\caption{Compute an optimal view for a guarded tree-shaped CQE instance}
\label{alg:view guarded datalog full}
\end{algorithm2e}

\begin{algorithm2e}%
\SetKwInOut{Input}{INPUT}
\SetKwInOut{Output}{OUTPUT}
\SetKwFor{ForEach}{for each}{do}{end}
\SetKw{Create}{create}
\SetKw{Mark}{mark}

\SetKw{IIf}{if}
\SetKw{FFor}{for}
\SetKw{EEach}{each}

\BlankLine
\AlgPhaseTop{Sub-routine \Anonymise}

\Input{a constant $a$}
\Output{a dataset $\View'$}

\BlankLine

$\View' := \View$;

$C := \set{A \mid A(a) \in \H}$;
	
$\sigma_a := \set{a}$;

\ForEach
{subset $\mathit{Sub}$ of $C$ closed under $\O_E$}
{%
 \Create a globally fresh copy $a_\mathit{Sub}$ of $a$;

 \If
 {$\O_E \cup \View' \cup \set{A(a_\mathit{Sub}) \mid A \in \mathit{Sub}} 
  \not\models P(\vec s)$ 
  for each $\vec s \in \cert{P}{\O}{\D}$}
 {$\View' := \View' \cup \set{A(a_\mathit{Sub}) \mid A \in \mathit{Sub}}$;
  
  $\sigma_a := \sigma_a \cup \set{a_\mathit{Sub}}$;}
}

\Return $\View'$;

\setcounter{AlgoLine}{0}

\AlgPhase{Sub-routine \AddRoles}

\Input{a binary atom $R(a,b)$}
\Output{a dataset $\View'$}

\BlankLine

$\View' := \View$;

\ForEach
{pair $a^* \in \sigma_a$ and $b^* \in \sigma_b$}
{\lIf
 {$\CheckRole(R(a^*,b^*), \View')$}
 {$\View' := \View' \cup \set{R(a^*,b^*)}$}
}

\Return $\View'$;

\setcounter{AlgoLine}{0}

\AlgPhase{Sub-routine \CheckRole}

\Input{a binary atom $R(a,b)$, a dataset $\View'$}
\Output{$\mathtt{True}$ or $\mathtt{False}$}

\BlankLine

\If
{$\O_E \cup \View' \cup \set{R(a,b)} \not\models P(\vec s)$ 
 for each $\vec s \in \cert{P}{\O}{\D}$\\
 \textbf{and} $\O_E \cup \View' \cup \set{R(a,b)} \models C(c)$
 implies $C(c) \in \View$ for any unary predicate $C$
}
{\Return $\mathtt{True}$;}

\lElse{\Return $\mathtt{False}$}

\caption{Sub-routines for Algorithm~\ref{alg:view guarded datalog full}}
\label{alg:view guarded datalog sub-routines 1}
\end{algorithm2e}

\clearpage
\section{Appendix (Reduction)}
\label{app:reduction}
In this section, we show the reduction of the problem 
	of uniform boundedness for binary Datalog
	to the problem of existence of optimal obstructions for Datalog CQE instances
	(see Section~\ref{sec:obstruction-optimal}).

Let $\O$ be a binary Datalog ontology over a signature $\Sigma$
	(observe that w.l.o.g.\ we can assume that $\O$ is connected).
Then, \O is \emph{uniformly bounded}
	if there is a constant $N$ 
	such that for every dataset $\D$ over $\Sigma$
	and for every ground atom $P(\vec t)$,
	if the atom has a proof from $\O$ and $\D$,
	then it has a proof not longer than $N$.
It is well known that 
	each relation $P(\vec x)$ defined by $\O$
	is equivalent to an infinite union of CQs
	$\bigvee_{i=1}^{\infty} \phi_i^P(\vec x)$.
Note that each $\phi_i^P(\vec x)$ is a result of applying some sequence of rules from $\O$
	to $P(\vec x)$.
Moreover,
\begin{itemize}
\item[(P1)] a Datalog ontology is uniformly bounded if and only if
	there exists a number $N$
	such that each $P(\vec x)$ is equivalent to $\bigvee_{i=1}^{N} \phi_i^P(\vec x)$.
\end{itemize}

\medskip
Now we are ready to provide the required reduction.
Let $\O$ be a binary Datalog ontology.
We are going to construct a CQE instance  $\Inst = (\O', \D, P)$
	which admits an optimal obstruction if and only if $\O$ is uniformly bounded.
The ontology $\O'$ of \Inst is defined as 
\begin{align*}
	\O & \cup \set{R_1^A(a, x) \land A(x) \rightarrow P 
		\mid A \text{ is unary and } A \in \Sigma}\\
	   & \cup \set{R_1^S(a, x_1) \land R_2^S(a, x_2) \land S(x_1, x_2) \rightarrow P
	   	\mid S \text{ is binary and } S \in \Sigma},
\end{align*}
	where all $R_1^A$ and $R_i^S$ and $P$ are fresh predicates.
The dataset $\D$ is equal to 
\[
	\set{A(a), S(a,a) \mid \text{$A$ is a unary and $S$ is a binary
	predicates from $\Sigma'$}} \cup \set{P},
\]
	where $\Sigma'$ is $\Sigma$ extended with fresh predicates $R_i^Q$.
Observe that this dataset admits any possible proof of $P$.

Let $\mathbb{Q} \setminus \mathbb{S}$ are built as in Definition~\ref{def:pseudo obstruction}.
It is easy to see that $\mathbb{Q} \setminus \mathbb{S}$
	contains the queries $\psi_i^A$ and $\psi_i^S$ of the form 
	$\exists x.R^A(a, x) \land \phi_i^A(x)$ and
	$\exists x_1 \exists x_2.R_1^S(a, x_1) \land R_2^S(a, x_2) \land \phi_i^S(x_1, x_2)$,
	respectively,
	for each $A, S \in \Sigma$
	as each of them with the help of $\O'$ compromises the policy.

Assume that $\O$ is not uniformly bounded;
	then, due to Property~(P1), there is some $Q \in \Sigma$
	such that for any number $N$
	we have that $\bigvee_{i=1}^{N} \phi_i^Q(\vec x) \not\equiv 
	\bigvee_{i=1}^{\infty} \phi_i^Q(\vec x)$.
That is, it is not the case that for each $\phi_i^Q$
	there exists $\phi_j^Q$ with $j \leq N$ such that 
	there is a homomorphism from $\phi_j^Q(\vec x)$ to $\phi_i^Q(\vec x)$ 
	(note that here distinguished variables are mapped into themselves).
This immediately yields that
	it is not the case that for each number $N$ and for each $\psi_i^Q$
	there exists $\psi_j^Q$ with $j \leq N$ such that
	there is a homomorphism from $\psi_j^Q$ to $\psi_i^Q$
	(note that, although here we do not have distinguished variables,
	we still have that the variables of $\psi_j^Q$
	that correspond to distinguished variables of $\phi_j^Q(\vec x)$
	are mapped to the variables of $\psi_i^Q$
	that correspond to distinguished variables of $\phi_i^Q(\vec x)$
	since they are ``marked'' by predicates $R_i^Q$).
Moreover, for every predicate $T$ different from $Q$,
	it holds that for any $i$ and any $j$
	there is no homomorphism from $\psi_i^Q$ to $\psi_j^T$
	since the former one mentions the predicate $R_1^Q$ and the latter one $R_1^T$.
Hence,
	there is no finite pseudo-obstruction for \Inst
	and therefore, due to Theorem~\ref{th:characterisation},
	no optimal obstruction censor for \Inst exists.

Assume that $\O$ is uniformly bounded and $N$ is a number such that
	for any dataset, if a fact can proved from $\O$ and the dataset,
	then there is a proof of this fact not longer than $N$.
Let $\mathbb{T}$ be a subset of $\mathbb{Q} \setminus \mathbb{S}$
	consisting of those Boolean CQs $\exists^* G$, where
	$G$ is a sub-goal in some proof of $P$ in $\O' \cup \D$
	of length not longer than $N + 3$.
We claim that the UCQ $U = \bigvee_{\phi \in \mathbb{T}} \phi$
	is an optimal obstruction for \Inst.
Assume that there exists a Boolean CQ $\psi = \exists^* G_0$
	from $\mathbb{Q} \setminus \mathbb{S}$
	with $G_0$ a sub-goal coming from some proof of length greater than $N+3$.
This means that $\O' \cup \mathbb{S} \cup \set{\psi} \models P$.
Than there exists a proof $\pi$ of $P$ from $\O' \cup \A$,
	where $\A = \freeze{\psi} \cup \bigcup_{\phi \in \mathbb{S}} \freeze{\phi}$,
	of length no longer than $N+3$
	($1$ step to apply one of the rules 
	$R_1^S(a, x_1) \land R_2^S(a, x_2) \land S(x_1, x_2) \rightarrow P$ from $\O'$,
	$N$ steps to proof $S(x_1, x_2)$ using rules from $\O \cup \A$,
	and $2$ additional steps to proof $R_1^S(a, b_1) \land R_2^S(a, b_2)$
	using facts from $\A$ for some elements $b_1$ and $b_2$).
W.l.o.g., we can assume that this proof is normalised.
Recall that all the rules that are applied after the frontier are from $\A$.
We can assume w.l.o.g.\ that rules from $\freeze{\psi}$ 
	are applied only at the very end of the proof.
Clearly, the goal $G$ right before we start to apply the rules from $\freeze{\psi}$
	is such that
\begin{inparaenum}[\it (i)]
\item $\exists^* G \in \mathbb{T}$ and
\item there is a homomorphism from $\exists^* G$ to $\psi$.
\end{inparaenum}
These properties imply that $\mathbb{T}$
	is a pseudo-obstruction and,
	since it is finite,
	by Theorem~\ref{th:characterisation} we have
	that an optimal obstruction censor for \Inst exists.


\begin{thebibliography}{}

\bibitem[\protect\citeauthoryear{Bao \bgroup \em et al.\egroup
  }{2007}]{DBLP:conf/webi/BaoSH07}
Jie Bao, Giora Slutzki, and Vasant Honavar.
\newblock {Privacy-Preserving Reasoning on the Semantic Web}.
\newblock In {\em WI}, pages 791--797, 2007.

\bibitem[\protect\citeauthoryear{Biskup and
  Bonatti}{2001}]{DBLP:journals/dke/BiskupB01}
Joachim Biskup and Piero Bonatti.
\newblock {Lying Versus Refusal for Known Potential Secrets}.
\newblock {\em Data Knowl. Eng.}, 38(2):199--222, 2001.

\bibitem[\protect\citeauthoryear{Biskup and
  Bonatti}{2004}]{DBLP:journals/ijisec/BiskupB04}
Joachim Biskup and Piero Bonatti.
\newblock {Controlled Query Evaluation for Enforcing Confidentiality in
  Complete Information Systems}.
\newblock {\em Int. J. Inf. Sec.}, 3(1):14--27, 2004.

\bibitem[\protect\citeauthoryear{Biskup and Bonatti}{2007}]{Biskup:2007bh}
Joachim Biskup and Piero Bonatti.
\newblock {Controlled Query Evaluation with Open Queries for a Decidable
  Relational Submodel}.
\newblock {\em Ann.\ Math.\ and Artif.\ Intell.}, 50(1-2):39--77, 2007.

\bibitem[\protect\citeauthoryear{Biskup and
  Weibert}{2008}]{DBLP:journals/ijisec/BiskupW08}
Joachim Biskup and Torben Weibert.
\newblock {Keeping Secrets in Incomplete Databases}.
\newblock {\em Int. J. Inf. Sec.}, 7(3):199--217, 2008.

\bibitem[\protect\citeauthoryear{Bonatti and
  Sauro}{2013}]{DBLP:conf/semweb/BonattiS13}
Piero Bonatti and Luigi Sauro.
\newblock {A Confidentiality Model for Ontologies}.
\newblock In {\em ISWC}, pages 17--32, 2013.

\bibitem[\protect\citeauthoryear{Bonatti \bgroup \em et al.\egroup
  }{1995}]{DBLP:journals/tkde/BonattiKS95}
Piero Bonatti, Sarit Kraus, and V.~S. Subrahmanian.
\newblock {Foundations of Secure Deductive Databases}.
\newblock {\em TKDE}, 7(3):406--422, 1995.

\bibitem[\protect\citeauthoryear{Calvanese \bgroup \em et al.\egroup
  }{2012}]{DBLP:journals/jcss/CalvaneseGLR12}
Diego Calvanese, Giuseppe {De Giacomo}, Maurizio Lenzerini, and Riccardo
  Rosati.
\newblock {View-based Query Answering in Description Logics: Semantics and
  Complexity}.
\newblock {\em J. Comput. Syst. Sci.}, 78(1):26--46, 2012.

\bibitem[\protect\citeauthoryear{{Cuenca Grau} and
  Horrocks}{2008}]{DBLP:conf/ecai/GrauH08}
Bernardo {Cuenca Grau} and Ian Horrocks.
\newblock {Privacy-Preserving Query Answering in Logic-based Information
  Systems}.
\newblock In {\em ECAI}, pages 40--44, 2008.

\bibitem[\protect\citeauthoryear{{Cuenca Grau} and
  Motik}{2012}]{DBLP:journals/jair/GrauM12}
Bernardo {Cuenca Grau} and Boris Motik.
\newblock {Reasoning over Ontologies with Hidden Content: The Import-by-Query
  Approach}.
\newblock {\em J. Artif. Intell. Res.}, 45:197--255, 2012.

\bibitem[\protect\citeauthoryear{{Cuenca Grau} \bgroup \em et al.\egroup
  }{2013}]{DBLP:conf/semweb/GrauKKZ13}
Bernardo {Cuenca Grau}, Evgeny Kharlamov, Egor~V. Kostylev, and Dmitriy
  Zheleznyakov.
\newblock {Controlled Query Evaluation over OWL 2 RL Ontologies}.
\newblock In {\em ISWC}, pages 49--65, 2013.

\bibitem[\protect\citeauthoryear{Deutsch and
  Papakonstantinou}{2005}]{DBLP:conf/icdt/DeutschP05}
Alin Deutsch and Yannis Papakonstantinou.
\newblock {Privacy in Database Publishing}.
\newblock In {\em ICDT}, pages 230--245, 2005.

\bibitem[\protect\citeauthoryear{Kolaitis and
  Vardi}{2008}]{DBLP:conf/dagstuhl/KolaitisV08}
Phokion~G. Kolaitis and Moshe~Y. Vardi.
\newblock {A Logical Approach to Constraint Satisfaction}.
\newblock In {\em Complexity of Constraints}, pages 125--155, 2008.

\bibitem[\protect\citeauthoryear{Konev \bgroup \em et al.\egroup
  }{2009}]{DBLP:conf/ijcai/KonevWW09}
Boris Konev, Dirk Walther, and Frank Wolter.
\newblock {Forgetting and Uniform Interpolation in Large-Scale Description
  Logic Terminologies}.
\newblock In {\em IJCAI}, pages 830--835, 2009.

\bibitem[\protect\citeauthoryear{Kontchakov \bgroup \em et al.\egroup
  }{2011}]{DBLP:conf/ijcai/KontchakovLTWZ11}
Roman Kontchakov, Carsten Lutz, David Toman, Frank Wolter, and Michael
  Zakharyaschev.
\newblock {The Combined Approach to Ontology-Based Data Access}.
\newblock In {\em IJCAI}, pages 2656--2661, 2011.

\bibitem[\protect\citeauthoryear{Lutz \bgroup \em et al.\egroup
  }{2009}]{DBLP:conf/ijcai/LutzTW09}
Carsten Lutz, David Toman, and Frank Wolter.
\newblock {Conjunctive Query Answering in the Description Logic EL Using a
  Relational Database System}.
\newblock In {\em IJCAI}, pages 2070--2075, 2009.

\bibitem[\protect\citeauthoryear{Lutz \bgroup \em et al.\egroup
  }{2013}]{DBLP:conf/semweb/LutzSTW13}
Carsten Lutz, Inan\c{c} Seylan, David Toman, and Frank Wolter.
\newblock {The Combined Approach to OBDA: Taming Role Hierarchies Using
  Filters}.
\newblock In {\em ISWC}, pages 314--330, 2013.

\bibitem[\protect\citeauthoryear{Marcinkowski}{1999}]{DBLP:journals/siamcomp/Marcinkowski99}
Jerzy Marcinkowski.
\newblock {Achilles, Turtle, and Undecidable Boundedness Problems for Small
  DATALOG Programs}.
\newblock {\em {SIAM} J. Comput.}, 29(1):231--257, 1999.

\bibitem[\protect\citeauthoryear{Miklau and
  Suciu}{2007}]{DBLP:journals/jcss/MiklauS07}
Gerome Miklau and Dan Suciu.
\newblock {A Formal Analysis of Information Disclosure in Data Exchange.}
\newblock {\em J. Comput. Syst. Sci.}, 73(3):507--534, 2007.

\bibitem[\protect\citeauthoryear{Motik \bgroup \em et al.\egroup
  }{2012}]{OWL2profiles}
Boris Motik, Bernardo {Cuenca Grau}, Ian Horrocks, Zhe Wu, Achille Fokoue, and
  Carsten Lutz.
\newblock {OWL 2 Web Ontology Language Profiles (2nd Edition)}, 2012.
\newblock {W3C} Recommendation.

\bibitem[\protect\citeauthoryear{Rizvi \bgroup \em et al.\egroup
  }{2004}]{DBLP:conf/sigmod/RizviMSR04}
Shariq Rizvi, Alberto~O. Mendelzon, S.~Sudarshan, and Prasan Roy.
\newblock {Extending Query Rewriting Techniques for Fine-Grained Access
  Control}.
\newblock In {\em SIGMOD}, pages 551--562. ACM, 2004.

\bibitem[\protect\citeauthoryear{Sandhu \bgroup \em et al.\egroup
  }{1996}]{DBLP:journals/computer/SandhuCFY96}
Ravi~S. Sandhu, Edward~J. Coyne, Hal~L. Feinstein, and Charles~E. Youman.
\newblock {Role-Based Access Control Models}.
\newblock {\em IEEE Computer}, 29(2):38--47, 1996.

\bibitem[\protect\citeauthoryear{Sicherman \bgroup \em et al.\egroup
  }{1983}]{DBLP:journals/tods/SichermanJR83}
George~L. Sicherman, Wiebren de~Jonge, and Reind~P. van~de Riet.
\newblock {Answering Queries Without Revealing Secrets}.
\newblock {\em ACM Trans. Database Syst.}, 8(1):41--59, 1983.

\bibitem[\protect\citeauthoryear{Stefanoni \bgroup \em et al.\egroup
  }{2013}]{DBLP:conf/aaai/StefanoniMH13}
Giorgio Stefanoni, Boris Motik, and Ian Horrocks.
\newblock {Introducing Nominals to the Combined Query Answering Approaches for
  EL}.
\newblock In {\em AAAI}, pages 1177--1183, 2013.

\bibitem[\protect\citeauthoryear{Stouppa and
  Studer}{2007}]{DBLP:conf/ershov/StouppaS06}
Phiniki Stouppa and Thomas Studer.
\newblock {A Formal Model of Data Privacy}.
\newblock In {\em PSI}, pages 400--408, 2007.

\bibitem[\protect\citeauthoryear{Studer and
  Werner}{2014}]{DBLP:journals/tdp/StuderW14}
Thomas Studer and Johannes Werner.
\newblock {Censors for Boolean Description Logic}.
\newblock {\em Trans.\ on Data Privacy}, 7(3):223--252, 2014.

\bibitem[\protect\citeauthoryear{Tao \bgroup \em et al.\egroup
  }{2010}]{DBLP:conf/rr/TaoSH10}
Jia Tao, Giora Slutzki, and Vasant Honavar.
\newblock {Secrecy-Preserving Query Answering for Instance Checking in
  {$\mathcal{EL}$}}.
\newblock In {\em RR}, pages 195--203, 2010.

\bibitem[\protect\citeauthoryear{Zhang and
  Mendelzon}{2005}]{DBLP:conf/icdt/ZhangM05}
Zheng Zhang and Alberto~O. Mendelzon.
\newblock {Authorization Views and Conditional Query Containment}.
\newblock In {\em ICDT}, pages 259--273, 2005.

\end{thebibliography}
  \end{document}